\documentclass{article}
\newif\ifneuripsstyle
\IfFileExists{neurips_2026.sty}{
    \neuripsstyletrue
    \usepackage[preprint]{neurips_2026}
}{
    \neuripsstylefalse
    \usepackage[margin=1in]{geometry}
    \usepackage{times}
    \usepackage[numbers]{natbib}
}
\usepackage{amsmath,amssymb,amsthm}
\usepackage{booktabs}
\usepackage{graphicx}
\usepackage{hyperref}
\usepackage{xcolor}
\usepackage{enumitem}
\usepackage{microtype}
\usepackage{bbm}
\usepackage{pifont}

\newtheorem{proposition}{Proposition}
\newtheorem{remark}[proposition]{Remark}

\title{TRIAGE: Role-Typed Credit Assignment for Agentic Reinforcement Learning}
\author{
    Yuanda Xu$^{1}$\thanks{Equal contribution. Hejian Sang's work was done while at LinkedIn Corporation.}\,
    \thanks{Correspondence to \texttt{yuanda@math.princeton.edu}} \quad
    Zhengze Zhou$^{1}$\footnotemark[1] \quad
    Hejian Sang$^{1}$\footnotemark[1] \quad
    Xiaomin Li$^{2}$ \quad \\
    {\bf Jiaxin Zhang$^{3}$ \quad Xinchen Du$^{1,4}$\thanks{Work done during an internship at LinkedIn Corporation.} \quad Sen Na$^{4}$ \quad Zhipeng Wang$^{1}$ \quad Alborz Geramifard$^{1}$} \\
    \vspace{0.2cm}
    $^{1}$LinkedIn Corporation \quad
    $^{2}$Harvard University \quad
    $^{3}$Johns Hopkins University \quad
    $^{4}$Georgia Institute of Technology
}

\date{}

\newcommand{\method}{TRIAGE}

\newcommand{\verifier}{V}
\newcommand{\role}{\rho}
\newcommand{\Role}{\mathcal{R}}

\begin{document}
\maketitle

\begin{abstract}
Agentic reinforcement learning requires assigning credit to environment-facing actions such as searches, clicks, edits, navigation commands, and object interactions. Standard GRPO uses the final verifier outcome as a uniform advantage over all action tokens. This outcome signal is useful but structurally incomplete: it punishes useful exploration in failed rollouts and reinforces redundant or regressive actions in successful rollouts. We propose \method{}, a role-typed credit assignment framework that adds a semantic role axis to outcome credit. A structured judge classifies each segment as decisive progress, useful exploration, no-progress infrastructure, or regression, and a fixed role-conditioned rule maps these labels to bounded segment-level process rewards. This keeps verifier outcomes as the source of optimization direction while correcting the two main blind spots of outcome-only credit. We further show that the Bayes-optimal role-measurable correction is the L2 projection of the per-segment advantage residual onto the role variable, and that \method{}'s fixed role constants approximate this projection, reducing advantage estimation error whenever the judge is reliable; we connect this to lower-variance policy gradients. Across ALFWorld, Search-QA, and WebShop, \method{} improves success rates over GRPO for two policy models and outperforms both a scalar judge-derived process reward and an outcome-supervised shared-backbone value baseline. Ablations show that the gain comes from role typing rather than merely adding dense rewards: reliable detection of regression inside successful trajectories is the dominant contributor, while exploration credit provides a consistent secondary gain; on completed ALFWorld and WebShop rollouts, \method{} also reduces environment-facing turns by an additional $10.4\%$ and $14.8\%$ relative to GRPO.
\end{abstract}

\section{Introduction}

Reinforcement learning with verifiable rewards has become a standard recipe for improving language-model reasoning and agentic behavior \citep{shao2024deepseekmath,deepseek2025r1,luong2024reft,yu2025dapo,sang2026beyond}. In Group Relative Policy Optimization (GRPO), a policy samples multiple trajectories for a prompt, receives final rewards from a verifier, and assigns relative advantages to the sampled outputs. This recipe is attractive because it requires no learned value model and optimizes directly against the deployment policy. However, when the output is an agentic trajectory rather than a single answer, the central credit-assignment question changes: \emph{which environment-facing actions deserve credit when supervision arrives only as a final verifier outcome?}

The unit of decision in this setting is not an arbitrary token span. It is an environment-facing segment: a search query, click, file edit, command, object interaction, or tool call that changes either the external state or the agent's information state. Across WebShop, Search-QA, and ALFWorld, such segments range from decisive actions (final purchases, answer submissions, object placements) to information-gathering actions (searches, inspections, reads) and low-value infrastructure (repeated navigation or redundant clicks).

\begin{figure}[t]
    \centering
    \includegraphics[width=\linewidth]{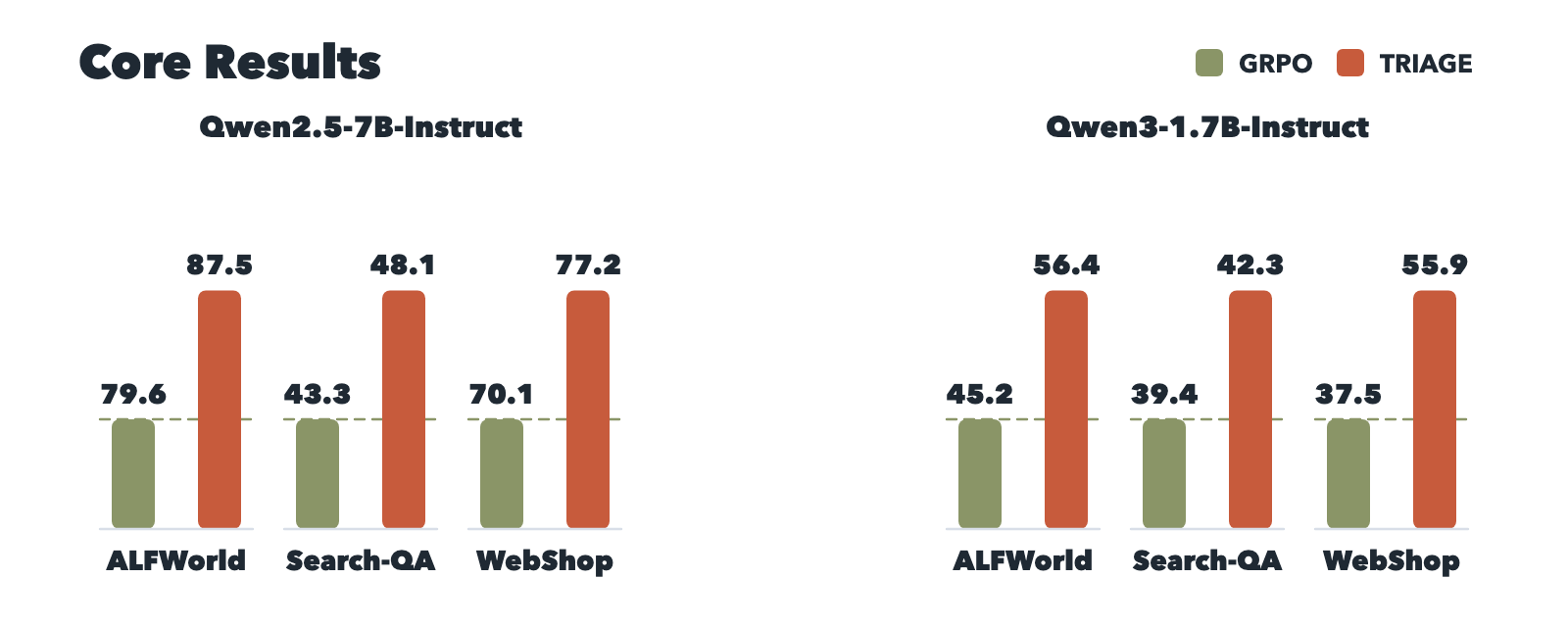}
    \caption{Core results. Across two policy models and three agentic benchmarks (ALFWorld, Search-QA, WebShop), \method{} consistently improves over the GRPO baseline (dashed line). Bar labels report mean success rate; vertical axes are truncated per panel to make differences visible.}
    \label{fig:core-results}
\end{figure}

Outcome credit is therefore useful but structurally incomplete. Standard GRPO treats all segments equally within a trajectory: if the trajectory succeeds, all action tokens are reinforced; if it fails, all are suppressed. This creates two systematic blind spots. First, failed rollouts can contain useful exploratory actions that should not inherit the full negative outcome credit. Second, successful rollouts can contain redundant or harmful actions that should not inherit positive credit merely because the agent later recovered. Final outcome tells us whether the trajectory solved the task, but it cannot say what local role each segment played.

Recent credit-assignment and on-policy supervision methods address parts of this problem. State-anchored estimators compare actions from matched states; process reward models learn dense progress signals; outcome-statistical methods estimate whether recurring segments concentrate in successful or failed rollouts; and token-importance methods reweight supervision within sampled outputs \citep{wang2025sparl,lu2026sdar,xu2026tip,sang2026crisp}. These approaches are useful, but they usually score each segment without specifying its semantic role: task progress, belief-state progress, harmless infrastructure, and regression should not receive the same credit rule. We test this distinction directly by comparing against two dense-signal controls---a scalar LLM process-reward baseline with the same judge and context window, and an outcome-supervised shared-backbone value baseline---so the empirical question is not whether dense segment rewards help, but whether role typing adds information beyond them.

Our central claim is therefore: \emph{agentic RL needs a role axis in addition to an outcome axis}. The most important distinction is that exploration is not no-progress. Exploration often has zero immediate task progress and may appear in both successful and failed trajectories. A purely outcome-statistical estimator can under-credit it because exploratory actions are not always success-specific. A generic process scorer can also conflate exploration with no-progress when no subgoal is completed immediately. Yet suppressing exploration is precisely how sparse-reward agent training becomes brittle: the policy learns to avoid information-gathering actions before it has enough information to act decisively.

We propose \method{}, a simple framework for role-aware credit estimation. Like medical triage, which sorts patients by the kind of attention they need before allocating treatment, \method{} first sorts each environment-facing segment into a semantic role before deciding how much credit it should inherit from the trajectory outcome. \method{} uses a structured LLM judge as a \emph{role classifier}, not as an unconstrained reward model. Given a bounded local context around each segment, the judge assigns one primary role: decisive progress, useful exploration, no-progress infrastructure, or regression. The RL algorithm then maps roles to different credit rules. Decisive progress receives strong outcome-aligned credit, useful exploration receives bounded positive credit, no-progress infrastructure is dampened toward zero, and regression is suppressed even when it appears in an otherwise successful trajectory.

This design deliberately separates semantic diagnosis from optimization direction. An LLM is well suited to answering local questions such as whether an action inspected a relevant file, narrowed a search, damaged state, or repeated known information. It is less suited to replacing the verifier. \method{} therefore keeps the GRPO outcome advantage as the base training signal and uses the role classifier only to add bounded process rewards or penalties at the segment level.

We make four contributions:
\begin{enumerate}[leftmargin=*]
    \item We identify two structural blind spots of outcome-only segment credit---useful exploration in failed rollouts and regression inside successful rollouts---and define a four-role taxonomy that adds a semantic role axis to trajectory-level outcome credit.
    \item We introduce \method{}, a role-conditioned credit assignment framework that uses a structured LLM judge for semantic role typing while keeping the GRPO outcome advantage as the source of optimization direction.
    \item We give a theoretical justification: the Bayes-optimal role-measurable correction is the L2 projection of the credit residual onto the role variable, and \method{}'s fixed role constants reduce advantage estimation error whenever aligned with this projection, connecting to unbiased, lower-variance policy gradients (Section~\ref{sec:theory}).
    \item We empirically evaluate \method{} across diverse agentic tasks and show consistent gains over GRPO, scalar judge-derived process rewards, and an outcome-supervised value baseline, while using manually labeled segments and role diagnostics to explain when the improvement comes from exploration retention, infrastructure damping, or regression suppression.
\end{enumerate}

\begin{figure}[t]
    \centering
    \includegraphics[width=1\linewidth]{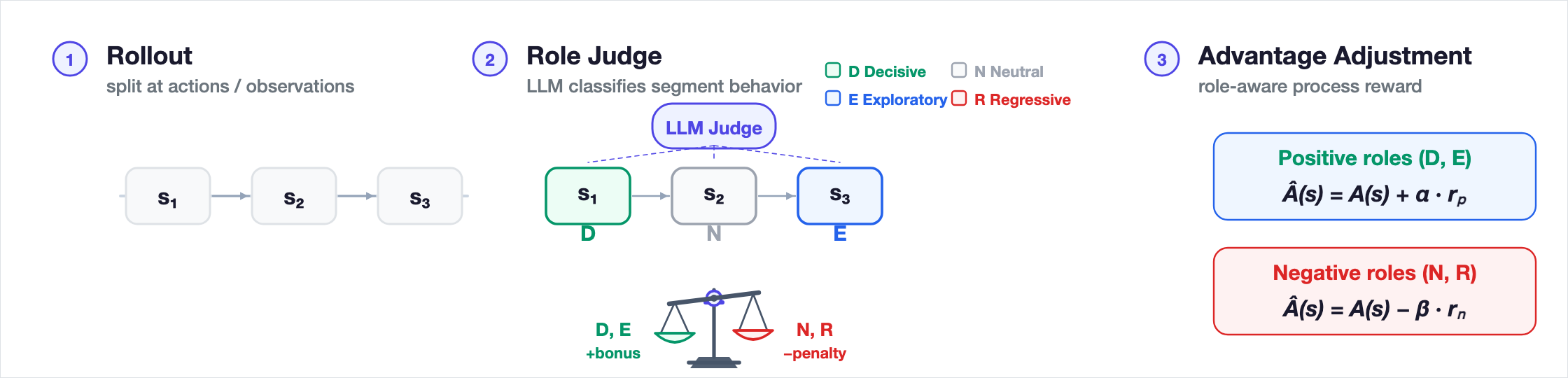}
    \caption{Overview of \method{}. Rollouts are split into environment-facing segments, a structured judge assigns semantic roles, and role-conditioned process rewards adjust segment-level GRPO advantages.}
    \label{fig:triage-overview}
\end{figure}

\section{Problem Setup: Segment Credit in Agentic RL}
\label{sec:setup}

\paragraph{GRPO.}
Given a task prompt $x$, GRPO samples $G$ trajectories, scores each with a verifier $r_i = \verifier(\tau_i)\in\{0,1\}$, and assigns the group-normalized advantage $A_i^{\mathrm{GRPO}} = (r_i - \bar{r})/(\sigma_r + \epsilon)$ uniformly to every token in the trajectory.
Some environment logs report raw success rewards on a different scale, such as 10 for success and 0 for failure; throughout training and in all equations, we binarize these raw rewards to $r_i\in\{0,1\}$.

\paragraph{From outcome credit to segment credit.}
An agentic trajectory $\tau_i = (a_{i,1}, o_{i,1}, \ldots, a_{i,K_i}, o_{i,K_i})$ consists of environment-facing action segments $a_{i,k}$ and their resulting observations $o_{i,k}$. Broadcasting a single $A_i^{\mathrm{GRPO}}$ to all segments treats a decisive purchase click, a useful diagnostic read, a harmless no-op, and a wrong edit identically. Process reward models offer one response by learning a dense value or progress score for each step \citep{lightman2023let}, but they do not by themselves specify whether a segment is exploration, infrastructure, or regression. Our goal is a segment-level advantage $A_{i,k}$ that reflects not only how good a segment is, but \emph{what role} it plays---which requires a structured label rather than a role-agnostic score.

\section{Why Outcome Credit Is Structurally Incomplete}
\label{sec:taxonomy}

Outcome credit supplies the correct trajectory-level direction, but it is a one-axis signal. It partitions rollouts into success and failure, then assigns all local decisions the same sign within each rollout. Agentic trajectories need a second axis: the local semantic role of each segment. Table~\ref{tab:outcome-role-conflicts} shows the two conflict cells that motivate this paper. A useful segment in a failed rollout should not be fully punished, and a regressive segment in a successful rollout should not inherit positive credit.

\begin{table}[t]
\centering
\caption{Outcome-only credit has two conflict cells. Final success or failure gives the optimization direction for the whole trajectory, but local segment roles determine whether a segment should inherit that direction unchanged.}
\label{tab:outcome-role-conflicts}
\footnotesize
\setlength{\tabcolsep}{5pt}
\begin{tabular}{@{}p{0.24\linewidth}p{0.33\linewidth}p{0.33\linewidth}@{}}
\toprule
Local segment role & Successful rollout & Failed rollout \\
\midrule
Useful local segment & should receive positive credit & should not be fully punished \\
Regressive local segment & should not inherit success credit & should be suppressed \\
\bottomrule
\end{tabular}
\end{table}

We instantiate this missing role axis with four segment types. Define a role variable
\begin{equation}
    \role_{i,k} \in \Role = \{D, E, N, R\},
\end{equation}
where $D$ denotes decisive progress, $E$ useful exploration, $N$ no-progress infrastructure, and $R$ regression. Table~\ref{tab:roles} gives the operational definition.

\begin{table}[t]
\centering
\caption{Role taxonomy}
\label{tab:roles}
\footnotesize
\setlength{\tabcolsep}{3pt}
\begin{tabular}{@{}p{0.16\linewidth}p{0.28\linewidth}p{0.28\linewidth}p{0.20\linewidth}@{}}
\toprule
Role & Definition & Examples & Default credit rule \\
\midrule
Decisive ($D$) & Produces verifier-checkable progress or completes a necessary subgoal & take target; buy item; correct answer & strong outcome-aligned \\
Exploration ($E$) & Reveals relevant state without immediate completion & read test; constrained search; inspect container & conditional positive \\
No-progress ($N$) & Changes neither task state nor belief state, but is harmless & duplicate click after completion; empty traversal & slightly penalized\\
Regression ($R$) & Corrupts state or repeats without information gain & wrong edit; wrong purchase; repeated examine/click & negative \\
\bottomrule
\end{tabular}
\end{table}

The taxonomy is intentionally not just an ordering by amount of progress. Exploration is not merely a small amount of progress. It is a different type of progress: it improves the information state rather than the environment state. This matters because many agent tasks are partially observable. Before editing a file, the agent must inspect relevant code and tests. Before buying an item, it must search and compare. Before manipulating an object, it may need to discover where the object or receptacle is. These actions should not be treated like repeated boilerplate just because they do not immediately satisfy the final verifier.

\paragraph{Role boundaries.}
The role boundaries are defined by what the segment changes. $D$ changes verifier-checkable task state: taking the target object, selecting the required item, submitting the correct answer, or applying the edit that makes a test pass. $E$ changes the information state without yet completing a subgoal: opening a container, reading a failing test, or running a targeted search. This boundary can be blurry in hindsight because an exploratory action may enable a later decisive one, but we reserve $D$ for direct task-state progress and use $E$ for first-time, reasonable information collection.

$N$ and $R$ cover the cases that should not receive positive progress credit. $N$ is harmless infrastructure that changes neither task state nor information state, such as an empty traversal or a generic command that does not affect the next decision. $R$ is locally harmful or redundant without information gain: a wrong edit, wrong purchase, corrupted object state, or repeated inspection/click after the relevant information is already known. Final outcome cannot resolve these distinctions. Useful exploration can appear in failed trajectories, and regression can appear in successful ones after later recovery, so role-aware credit must judge the local segment rather than only its trajectory-level success label.

\paragraph{What the judge must get right.}
The judge does not need perfect $D/E$ boundary agreement. Its key capability is \emph{asymmetric error correction}: in successful rollouts, find local regressions that should not inherit positive credit; in failed rollouts, find locally useful segments that should not inherit full negative credit. Operationally, regression has two subclasses: \emph{state corruption} (wrong edit, wrong purchase, wrong object) and \emph{redundant-without-information-gain} (repeated inspection or click after the information is already known).

\paragraph{Implications for diagnostics.}
The taxonomy also determines what we measure experimentally. Useful exploration is outcome-mixed: it appears in both successful and failed rollouts, so outcome association can make it look neutral or negative. No-progress infrastructure receives nonzero advantage under uniform broadcasting, wasting gradient on boilerplate actions. Regression can appear inside successful trajectories after later recovery, so final outcomes hide local harm. We therefore track three diagnostics in the experiments: exploration retention, infrastructure damping, and regression suppression.

\section{\method{}: Role-Conditioned Segment Credit}
\label{sec:method}

\method{} has two components: a structured role judge and a role-conditioned process reward. The policy update remains the standard GRPO update. Rather than using the LLM judge as an unconstrained scalar reward model, \method{} uses a rubric-guided judge to assign one auditable semantic role per segment, and maps those roles to fixed credit rules. The only change is the advantage assigned to each environment-facing segment: we keep the trajectory-level GRPO advantage and add a bounded process reward whose form depends on the segment role.

\paragraph{Role-judge context window.}
The training-time role judge uses a bounded local context window around each segment; in our experiments this window includes up to five previous and five future action--observation pairs. Appendix~\ref{app:judge-context-window} gives the exact window definition. The judge does not receive the final verifier outcome.

Let $A_i^{\mathrm{GRPO}}$ be the outcome advantage for trajectory $i$. For segment $a_{i,k}$, \method{} defines
\begin{equation}
    A_{i,k}^{\method} = A_i^{\mathrm{GRPO}} + \lambda c_{\hat{\rho}_{i,k}},
\end{equation}
where $c_{\hat{\rho}_{i,k}}$ is a fixed process reward for the assigned role and $\lambda$ controls how strongly this local signal is mixed into the GRPO advantage. The auxiliary judge scores are used only to help choose the role label, not as additional training-time notation.

A simple instantiation sets
\begin{equation}
    (c_D,c_E,c_N,c_R)=(1,0.5,-0.1,-0.5).
\end{equation}
Thus decisive progress receives a unit process reward, useful exploration receives a smaller positive reward, no-progress infrastructure receives only a small step cost, and regression receives a larger local penalty even if the trajectory succeeds. This scale follows the usual agent-RL convention that task progress is around $+1$, harmless inefficiency receives a mild penalty around $-0.1$, and clearly unhelpful actions receive a stronger negative reward. This keeps the main comparison close to GRPO: the dominant signal is still the outcome advantage, while role typing adds only a bounded segment-level process reward.

Unless otherwise stated, we use $\lambda=0.4$ for Search-QA and $\lambda=0.2$ for the other two environments, keeping $(c_D,c_E,c_N,c_R)=(1,0.5,-0.1,-0.5)$ fixed across tasks. The role constants are never tuned; the only tuned hyperparameter is $\lambda$, selected on the training split by training success rate with the test set held out for final evaluation. The $\lambda\times|c_R|$ grids in Appendix~\ref{app:sensitivity} are \emph{post-hoc} sensitivity analyses and were not used to choose $\lambda$.

For stability, the resulting segment advantages are whitened within each batch before being broadcast to segment tokens:
\begin{equation}
    \tilde{A}_{i,k}^{\method} = \frac{A_{i,k}^{\method} - \mu_{\mathcal{B}}}{\sigma_{\mathcal{B}}+\epsilon}.
\end{equation}
The policy update is the usual clipped GRPO objective with $\tilde{A}_{i,k}^{\method}$ assigned to tokens belonging to segment $k$.
In the evaluated environments, a segment coincides with the standard environment step used in prior agent-RL work: one admissible ALFWorld command, one WebShop \texttt{search[...]} or \texttt{click[...]} action, or one Search-QA $\langle$search$\rangle$ query or final $\langle$answer$\rangle$ submission. The segment advantage is applied only to generated tokens in the corresponding environment-facing turn; prompt and observation tokens are excluded from the policy loss.

\paragraph{Training procedure.}
In each GRPO batch, we first compute the usual trajectory advantage $A_i^{\mathrm{GRPO}}$. We then split each rollout into environment-facing action segments and ask the role judge for the segment role and auxiliary scores $(q,u,h,b)$. The role-conditioned process reward is added to the GRPO advantage, the resulting segment advantages are normalized within the batch, and each normalized value is broadcast to the tokens in that segment before the standard clipped GRPO update. No judge is used at evaluation time.

\subsection{Theoretical Justification: Role Conditioning as an L2 Projection of the Credit Residual}
\label{sec:theory}

We give a justification, not a guarantee: under a stated sufficiency assumption, the Bayes-optimal role-measurable correction is the L2 projection of the credit residual onto the role variable, and the fixed constants used by \method{} inherit a strictly smaller estimation error than uniform broadcasting whenever aligned with this projection. We connect this to lower-variance policy gradients and flag where the assumption fails in Appendix~\ref{app:theory-extended}; all proofs are in Appendix~\ref{app:proofs}.

\paragraph{Setup.}
Let $A_{i,k}^*$ denote the (unobserved) oracle per-segment advantage and let $A_i^{\mathrm{GRPO}}$ be the trajectory advantage that GRPO broadcasts to every segment. Define the \emph{credit residual}
\begin{equation}
    \delta_{i,k} \triangleq A_{i,k}^* - A_i^{\mathrm{GRPO}},
\end{equation}
the within-trajectory variation in true credit that uniform broadcasting discards. A segment-level estimator that adds a correction $g$ to $A_i^{\mathrm{GRPO}}$ incurs squared error $\mathbb{E}\big[(A_i^{\mathrm{GRPO}} + g - A_{i,k}^*)^2\big] = \mathbb{E}\big[(g - \delta_{i,k})^2\big]$.

\begin{proposition}[Optimal role-measurable correction]
\label{prop:projection}
Among all corrections $g(\rho)$ that are measurable with respect to the segment role $\rho_{i,k}$, the minimizer of the segment-advantage MSE is the conditional expectation of the residual,
\begin{equation}
    g^\star(\rho) = \mathbb{E}\big[\delta_{i,k} \,\big|\, \rho_{i,k}=\rho\big],
\end{equation}
and the resulting MSE reduction relative to GRPO is
\begin{equation}
    \mathrm{MSE}^{\mathrm{GRPO}} - \mathrm{MSE}^{g^\star} = \mathbb{E}\Big[\big(\mathbb{E}[\delta_{i,k}\mid\rho_{i,k}]\big)^2\Big] \;\ge\; 0.
\end{equation}
\end{proposition}

Proposition~\ref{prop:projection} formalizes the paper's central claim: role labels help \emph{exactly to the extent that they explain nonzero credit residual}, i.e. whenever $\mathbb{E}[\delta\mid\rho]\neq 0$ for some role. The four-role taxonomy is thus an interpretable, coarse discretization of the Bayes-optimal correction $g^\star$, with $g^\star(R)<0$ (regression is over-credited by broadcasting) and $g^\star(E)>0$ in failed rollouts (exploration is over-punished)---precisely the two conflict cells of Table~\ref{tab:outcome-role-conflicts}.

\method{} uses fixed role constants rather than estimating $g^\star$. For the correction $\lambda c_{\hat\rho}$, the MSE change relative to GRPO is
\begin{equation}
    \Delta_{\mathrm{MSE}} = \lambda^2 \mathbb{E}[c_{\hat\rho}^2] - 2\lambda\,\mathrm{Cov}(c_{\hat\rho},\delta),
\end{equation}
so any positively aligned role signal reduces error for sufficiently small $\lambda$. This is exactly the desired sign pattern: negative for regression that GRPO over-credits and positive for exploration that GRPO over-punishes. Appendix~\ref{app:theory-extended} gives the full fixed-constant condition, connects the correction to policy-gradient variance, and states the failure modes.

\section{Experiments}
\label{sec:experiments}

We design experiments to test role-aware credit rather than merely final performance. The central empirical question is whether \method{} preserves useful exploration while suppressing no-progress and regression.

\subsection{Experimental Setup}

\paragraph{Environments.}
We evaluate on three families of agentic tasks. ALFWorld tests embodied household planning with templated actions \citep{shridhar2021alfworld}. Search-QA tests multi-turn retrieval and answer generation, where query formulation and evidence gathering are exploratory \citep{dunn2017searchqa}. WebShop tests product search and purchase \citep{yao2022webshop}, where search/filter actions are exploratory and the purchase action is decisive.

\paragraph{Models and training.}
We evaluate Qwen2.5-7B-Instruct and Qwen3-1.7B-Instruct as deployable student policies for all three environments \citep{qwen2024qwen25}. Training uses GRPO with $G$ rollouts per prompt, implemented on top of the verl framework \citep{sheng2024hybridflow}. \method{} uses the same rollouts and verifier rewards as GRPO, plus cached role labels from an LLM judge. All final evaluations use the unaided deployment policy without judge calls. For ALFWorld and WebShop, we repeat training and evaluation with ten independent runs and report mean $\pm$ sample standard deviation. Search-QA runs are substantially more expensive because each optimization step requires large-model rollout with multi-turn retrieval and verifier evaluation, so Search-QA results are reported from a single run under the same fixed training configuration; consequently Search-QA entries in the tables do not include a standard deviation. The judge is used only during training; inference uses the unaided policy, so \method{} and GRPO share identical evaluation-time compute. A matched-wall-clock comparison that extends GRPO to the same total training-time compute as \method{} is reported in Appendix~\ref{app:training-hyperparameters}.

\subsection{Main Results}

\begin{table}[t]
\centering
\caption{Main results: success rate (\%). ALFWorld and WebShop entries with $\pm$ are mean $\pm$ sample standard deviation over ten independent training-and-evaluation runs. Search-QA is reported as a single run because the retrieval-augmented rollout loop makes repeated full training runs substantially more expensive. The ``no evidence'' rows use the same Qwen3-8B-thinking judge but with a prompt that does not ask for a per-segment evidence string and only requests the role label (see Appendix~\ref{app:judge-audit} for the full default prompt that does require evidence).}
\label{tab:main-results}
\footnotesize
\setlength{\tabcolsep}{4pt}
\begin{tabular}{llccc}
\toprule
Model & Method & ALFWorld & Search-QA & WebShop \\
\midrule
Qwen2.5-7B-Instruct & GRPO & $79.6 \pm 1.9$ & 43.3 & $70.1 \pm 2.3$ \\
% & SDAR (reported) & 85.9 & 49.0 & 82.8 \\
 & \method{} w/ Qwen3-8B no-think judge & $76.8 \pm 2.1$ & 45.0 & $65.4 \pm 2.8$ \\
 %& \method{} w/ 25\% corrupted Qwen3-8B-thinking labels & -- & -- & -- \\
 & \method{} w/ Qwen3-8B-thinking judge, no evidence & $83.1 \pm 2.6$ & 46.4 & $73.5 \pm 2.5$ \\
 & \method{} w/ Qwen3-8B-thinking judge & $87.5 \pm 2.4$ & $48.1$ & $77.2 \pm 2.2$ \\
\midrule
Qwen3-1.7B-Instruct & GRPO & $45.2 \pm 2.5$ & 39.4 & $37.5 \pm 2.8$ \\
 & \method{} w/ Qwen3-8B no-think judge & $40.7 \pm 2.1$ & 40.2 & $35.1 \pm 2.2$ \\
 & \method{} w/ Qwen3-8B-thinking judge, no evidence & $51.8 \pm 1.8$ & 41.1 & $49.6 \pm 1.9$ \\
 & \method{} w/ Qwen3-8B-thinking judge & $56.4 \pm 1.2$ & 42.3 & $55.9 \pm 1.5$\\
\bottomrule
\end{tabular}
\end{table}
Figure~\ref{fig:core-results} summarizes the main comparison, and Table~\ref{tab:main-results} reports the underlying numbers. With the default Qwen3-8B-thinking judge, \method{} improves over GRPO on all three benchmarks for both policies, with the largest gains on ALFWorld and WebShop---the two audited environments with the highest regression mass (48\% and 43\%; Appendix~\ref{app:role-audit}). The Search-QA gain is smaller but consistent, matching its more exploration-dominated, lower-regression profile. This pattern is what role-conditioned credit predicts: most of the benefit comes from withholding positive credit from regressive segments that vanilla GRPO reinforces whenever the trajectory happens to succeed.

The comparison also shows that the benefit depends on judge reliability rather than on simply adding a dense reward. Substituting the Qwen3-8B \emph{no-think} judge---which collapses on the $R$-in-success cell (Table~\ref{tab:judge-audit-f1})---drives \method{} \emph{below} the GRPO baseline on ALFWorld and WebShop for both policies, confirming that the gains stem from accurate role typing and not from the extra reward term alone. Removing the evidence requirement (``no evidence'' rows) keeps \method{} above GRPO but consistently trails the default prompt, so thinking is necessary for the hard $R$-in-success cell while structured evidence acts as a low-cost calibration knob on top of it.

\subsection{Does the Judge Recover the Conflict Cells?}

Because \method{} relies on a role judge, we audit whether the judge recovers local segment roles rather than simply echoing the final outcome. Two annotators independently label 135 environment-facing segments from 18 logged trajectories (3 ALFWorld, 3 WebShop, 12 Search-QA), reaching \textbf{88.1}\% raw agreement; disagreements are adjudicated by a senior annotator and used as ground truth. The prompt, labels, and examples are in Appendix~\ref{app:judge-audit}.

Table~\ref{tab:judge-audit-f1} reports binary F1 by role--outcome cell, focusing on the two conflict cells: $R$ inside successful rollouts and $E$ inside failed rollouts. We omit $D$ in failed rollouts because it has zero support in this labeled set.

\begin{table}[t]
\centering
\caption{Qwen3 role judge F1 (\%) across 135 labeled segments, split by hand-labeled role and trajectory outcome. Column counts give the number of positive examples in each cell.}
\label{tab:judge-audit-f1}
\footnotesize
\setlength{\tabcolsep}{4pt}
\begin{tabular}{lccccc}
\toprule
Config & $R$ in success rollouts & $R$ in failed rollouts & $E$ in success rollouts & $E$ in failed rollouts & $D$ in success rollouts \\
    & $(n=35)$ & $(n=20)$ & $(n=29)$ & $(n=21)$ & $(n=25)$ \\
\midrule
8B no-think  & 29.2 & 81.1 & 56.1 & 90.0 & 55.6 \\
8B think & 86.1 & \textbf{91.9} & \textbf{78.7} & \textbf{95.2} & 65.1 \\
14B no-think &  5.7 & 80.0 & 54.0 & 90.9 & 62.9 \\
14B think    & 72.7 & 86.5 & 70.8 & 90.9 & 56.4 \\
32B no-think & 35.9 & 74.3 & 56.4 & 82.6 & \textbf{73.7} \\
32B think    & \textbf{88.6} & 83.3 & 70.8 & 88.9 & 65.1 \\
\bottomrule
\end{tabular}
\end{table}

The result supports the two-blind-spot framing. Thinking is not uniformly useful; its large effect is concentrated in $R$-in-success, where it raises F1 from roughly 24 to 82 averaged over model sizes. The easy cell is $E$-in-failure (F1 $>82$ even without thinking); the hard cell is finding regression exactly where the verifier says the rollout succeeded. Scaling helps less than enabling thinking: 8B-thinking is within three F1 points of 32B-thinking on $R$-in-success at substantially lower inference cost. We therefore use Qwen3-8B with thinking enabled as the default judge.

\subsection{Comparisons and Ablations}
All comparisons and ablations in this section use Qwen2.5-7B-Instruct. We organize the analysis around three questions: how \method{} compares with stronger credit-assignment baselines, whether role typing adds value beyond generic dense process rewards, and whether the trained policy exhibits the intended behavioral changes.

\paragraph{External credit-assignment baselines.}
Table~\ref{tab:baseline-comparison} situates \method{} against stronger credit-assignment baselines reproduced under an identical protocol: PPO with a learned critic, GiGPO, which assigns step-level credit by grouping actions from recurring states \citep{feng2025gigpo}, and a shared-backbone value baseline that learns a dense per-segment signal from the same verifier rewards. \method{} improves over PPO on all three benchmarks without a separate value network. Relative to GiGPO, \method{} is higher on WebShop and statistically tied on ALFWorld, while GiGPO does not apply to Search-QA because its state grouping degenerates when per-step states almost never recur. Relative to the value baseline, \method{} tests the central claim of the paper: dense segment credit alone is not enough when productive and regressive actions have similar outcome-trained values, and the missing information is the segment's semantic role. The key difference is signal source: GiGPO derives micro-advantages \emph{structurally} from recurring states, the value baseline derives them \emph{statistically} from outcome regression, and \method{} derives them \emph{semantically} from role labels---targeting the conflict cells that role-agnostic dense signals cannot resolve.

\begin{table}[t]
\centering
\caption{Comparison with stronger credit-assignment baselines on Qwen2.5-7B-Instruct: success rate (\%). All methods are our own runs under an identical protocol (see Appendix~\ref{app:value-baseline} for the shared-backbone value baseline). The GRPO and \method{} rows are repeated from Table~\ref{tab:main-results} for reference. GiGPO's Search-QA entry is left blank because its step-level state grouping degenerates to episode-level GRPO when per-step states embed retrieved documents that almost never recur across rollouts.}
\label{tab:baseline-comparison}
\footnotesize
\setlength{\tabcolsep}{4pt}
\begin{tabular}{lccc}
\toprule
Method & ALFWorld & Search-QA & WebShop \\
\midrule
GRPO & $79.6 \pm 1.9$ & 43.3 & $70.1 \pm 2.3$ \\
PPO & $81.7 \pm 2.1$ & 45.3 & $71.5 \pm 2.0$ \\
GiGPO & $87.8 \pm 2.2$ & -- & $74.3 \pm 2.9$ \\
% PPO (reported by \citet{feng2025gigpo}) & $80.4 \pm 5.1$ & -- & $65.7 \pm 4.0$ \\
% GiGPO (reported by \citet{feng2025gigpo}) & $\mathbf{90.8 \pm 2.2}$ & -- & $72.8 \pm 3.2$ \\
Shared-backbone value baseline (App.~\ref{app:value-baseline}) & $85.2\pm 2.7$ & 46.8 & $70.8 \pm 3.7$ \\
\method{} w/ Qwen3-8B-thinking judge & $87.5 \pm 2.4$ & $\mathbf{48.1}$ & $\mathbf{77.2 \pm 2.2}$ \\
\bottomrule
\end{tabular}
\end{table}
The shared-backbone value baseline improves over GRPO on the two longer-rollout environments (ALFWorld $79.6\rightarrow85.2$, $+5.6$; Search-QA $43.3\rightarrow46.8$, $+3.5$), confirming that a learned dense per-segment baseline trained on the same verifier reward is a meaningful upgrade over uniform broadcast. On WebShop, however, it barely moves ($70.1\rightarrow70.8$, within run-to-run variance), while \method{} reaches $77.2$. The reason is structural: WebShop regressions are repeated clicks of an already-selected attribute that leave the observation almost unchanged, so an outcome-trained value head cannot separate the productive click from its redundant repeat, whereas the role classifier reads the action history and labels the repeat $R$. Appendix~\ref{app:value-baseline} gives the full analysis. \method{}'s per-step wall-clock is higher than GRPO's because of judge calls; extending GRPO to matched training wall-clock (Appendix~\ref{app:training-hyperparameters}, Table~\ref{tab:matched-compute}) narrows but does not close the gap, indicating that the gain is not simply extra compute.
\paragraph{Role-reward ablations.}

We also include a scalar process-reward baseline to separate the value of role typing from the value of adding any judge-derived dense reward. This baseline uses the same Qwen3-8B-thinking judge and the same local context window as \method{}, but asks for a single progress score $s_{i,k}\in[-1,1]$ rather than a discrete role. We add this score to the GRPO advantage as
\begin{equation}
        A_{i,k}=A_i^{\mathrm{GRPO}}+\lambda s_{i,k},
\end{equation}
and apply the same batch whitening as \method{}. This controls for judge access, local context, and dense reward shaping while removing role-conditioned credit rules. Thus the comparison isolates whether the advantage comes from a generic process reward or from the role-specific mapping that treats exploration, no-progress infrastructure, and regression differently.

  \begin{table}[t]
  \centering                     
\caption{Ablation results on Qwen2.5-7B-Instruct: success rate (\%). We test role-reward components, focusing on the frequent/high-impact failure modes: regression and exploration.}
  \label{tab:ablations}                
  \footnotesize
  \setlength{\tabcolsep}{4pt}
  \begin{tabular}{@{}p{0.32\linewidth}p{0.26\linewidth}ccc@{}}                            
  \toprule  
  Ablation & Change & ALFWorld & Search-QA & WebShop \\             
  \midrule  
  Raw GRPO & no role judge or process reward & $79.6 \pm 1.9$ & 43.3 & $70.1 \pm 2.3$ \\         
    Scalar process reward & 8B-thinking judge, no role typing & $84.8 \pm 2.8$ & 45.9 & $72.1 \pm 2.8$ \\
  No regression penalty & set $c_R=0$ & $81.4 \pm 2.6$ & 46.7 & $73.1 \pm 2.3$ \\                      
  No exploration bonus & set $c_E=0$ & $85.8 \pm 2.4$ & 47.5 & $75.5 \pm 2.2$ \\           
  \method{} & none & $\mathbf{87.5 \pm 2.4}$ & $\mathbf{48.1}$ & $\mathbf{77.2 \pm 2.2}$ \\         
  \bottomrule                  
  \end{tabular}                   
  \end{table}   

Table~\ref{tab:ablations} isolates the two role-reward components and the role-typing effect itself. The scalar process-reward baseline improves over GRPO, confirming that dense segment feedback is useful, but it remains below \method{} on every benchmark. Removing either role component further degrades \method{}, so the gain is not an artifact of simply adding a dense reward from the same judge. The regression penalty ($c_R$) is the dominant contributor: zeroing it costs $1.8$--$6.1$ points across benchmarks and leaves ALFWorld and WebShop only marginally above raw GRPO. The exploration bonus ($c_E$) provides a smaller but consistently positive top-up ($0.6$--$1.7$ points). This ordering matches the role audit: ALFWorld and WebShop carry regression mass of $\approx 48\%$ and $\approx 43\%$ (Appendix~\ref{app:role-audit}), so most of \method{}'s gain comes from suppressing $R$ credit inside successful trajectories. Consistent with this mechanism, \method{} also reduces completed-rollout length by $10.4\%$ and $14.8\%$ relative to GRPO on the two environments (Appendix~\ref{app:rollout-length}). \method{} is stable to the role-constant scale and $\lambda$ within a reasonable range (Appendix~\ref{app:sensitivity}).

\section{Discussion and Limitations}
\paragraph{Limitations.}
Role labels are semantic estimates, not ground truth. A judge can overvalue plausible exploration, miss subtle regressions, or rely too much on final outcomes. \method{} mitigates this by using the judge only for structured role diagnosis and keeping verifier outcomes as the base optimization signal, but it does not remove judge error.

Role usefulness is also context-dependent. The same search, read, or test command can be informative once and redundant later, so the classifier must condition on local state and redundancy rather than action strings alone. Finally, role-aware credit is not causal identification: it improves local attribution, but counterfactual environment interventions would be needed to prove that a segment was necessary.

\paragraph{Future work.}
This paper uses one primary role per segment to keep the signal auditable. A natural extension is a soft role distribution, e.g., $(p_D,p_E,p_N,p_R)$, with credit computed as an expectation under role-specific constants. This could better represent mixed segments, such as a search that reveals useful evidence while also introducing distractors, but it would require reliable calibration and stronger audit procedures.

\method{} is also compatible with segment bucketing and outcome-statistical estimators. Bucketing can decide which segments share statistical evidence, while role labels decide how that evidence should be interpreted. Combining the two is a promising direction for domains where exact action arguments are sparse and repeated segments are rare.

Finally, the discrete four-role label is only the first layer of role-aware judging. On harder tasks or stronger base agents, obvious loops, wrong purchases, and repeated inspections become rare, and the credit problem shifts from detecting coarse failures to estimating how much each segment advances the task or belief state. In that regime the same framework can use a stronger judge to assign finer-grained process rewards \emph{within} each role rather than a single discrete label.

\section{Related Work}

\begin{table}[t]
\centering
\caption{Where \method{} sits among agentic credit-assignment methods. \emph{Expl.\,$\neq$\,no-prog.}: separates useful exploration from harmless no-progress; \emph{Regr.\,in success}: can withhold credit from regressive steps inside successful rollouts; \emph{No state match}: works without recurring or matchable states.}
\label{tab:credit-design-space}
\footnotesize
\setlength{\tabcolsep}{6pt}
\begin{tabular}{@{}lcccc@{}}
\toprule
Method family & Granularity & \shortstack{Expl.\,$\neq$\\no-prog.} & \shortstack{Regr.\,in\\success} & \shortstack{No state\\match} \\
\midrule
Outcome / group RL & trajectory & \textcolor{red!65!black}{\ding{55}} & \textcolor{red!65!black}{\ding{55}} & \textcolor{green!55!black}{\checkmark} \\
GiGPO \citep{feng2025gigpo} & step & \textcolor{red!65!black}{\ding{55}} & partial & \textcolor{red!65!black}{\ding{55}} \\
Step / process rewards \citep{wang2025sparl,lightman2023let} & step & \textcolor{red!65!black}{\ding{55}} & \textcolor{red!65!black}{\ding{55}} & \textcolor{green!55!black}{\checkmark} \\
\method{} (ours) & step & \textcolor{green!55!black}{\checkmark} & \textcolor{green!55!black}{\checkmark} & \textcolor{green!55!black}{\checkmark} \\
\bottomrule
\end{tabular}
\end{table}

\paragraph{Agentic credit assignment.}
Agentic RL requires assigning credit across environment-facing decisions rather than only across tokens. Table~\ref{tab:credit-design-space} summarizes the closest design choices. State-anchored methods such as GiGPO compare actions taken from matched states \citep{feng2025gigpo}; stepwise progress and process-reward methods learn scalar dense scores for intermediate steps \citep{wang2025sparl,lightman2023let}. \method{} is complementary: it keeps the outcome advantage but adds a semantic role label, so the update can distinguish useful exploration from no-progress behavior and regression from ordinary low progress.

\paragraph{Process reward models and LLM judges.}
Process reward models provide dense supervision by scoring intermediate reasoning or agent steps \citep{lightman2023let}. LLM-as-judge methods can evaluate generated outputs, critique trajectories, or assign rubric scores \citep{shinn2023reflexion,madaan2023selfrefine,fang2026rubricopd,lan2026aillawbibliometric}. Unstructured process scores can be brittle: they may punish correct actions in failed trajectories, over-credit plausible narration, or conflate exploration with lack of progress. \method{} uses the judge more narrowly as a structured classifier over segment roles. This reduces the burden on the judge and makes the resulting signal easier to audit.

\paragraph{Exploration in language agents.}
Language agents often rely on information-gathering actions such as search, inspect, read, and test execution \citep{yao2023react,schick2023toolformer}. Related prompting and self-improvement methods also exploit multiple sampled reasoning paths, search trees, or self-generated rationales to expose useful intermediate information \citep{wang2022selfconsistency,yao2023tree,zelikman2022star}. These actions change the agent's belief state rather than immediately completing the task. In sparse-reward RL, such actions are easy to misclassify as neutral or wasteful. \method{} makes belief-state progress an explicit credit category, allowing training to preserve useful exploration while still suppressing redundant or irrelevant exploration.

\paragraph{On-policy distillation and token weighting.}
On-policy distillation and token-importance methods refine supervision on sampled trajectories \citep{xu2026tip,sang2026beyond,sang2026crisp,agarwal2024onpolicy,zhou2023causal,yu2026dismantling}. These methods mostly operate at token or response granularity. \method{} operates at the agentic segment level and can be applied to either RL advantages or distillation losses: role labels can gate which action turns receive strong distillation or reinforcement.

\section{Conclusion}

We argued that agentic credit assignment requires distinguishing what role each environment-facing segment plays. The key missing distinction is that exploration is not no-progress: an action can improve the agent's belief state without immediately completing a subgoal. \method{} operationalizes this idea with a structured role judge and role-conditioned credit rules, keeping the GRPO outcome advantage as the optimization direction while adding a bounded, role-typed correction. Across ALFWorld, Search-QA, and WebShop, this lifts success rates over GRPO for two policy models---by up to $7.9$ points on Qwen2.5-7B and $18.4$ on Qwen3-1.7B---and shortens completed rollouts by $10.4\%$--$14.8\%$, with ablations and a manual role audit confirming that suppressing regression inside successful trajectories is the dominant source of the gain. Theoretically, the Bayes-optimal role-measurable correction is the L2 projection of the credit residual onto the role variable, and \method{}'s fixed constants approximate this projection, so the benefit is tied directly to judge reliability, which our audit measures rather than assumes. By reinforcing decisive progress, preserving useful exploration, damping no-progress infrastructure, and suppressing regression, \method{} offers a principled path toward sparse-reward RL for agents whose success depends on information gathering and recovery.

\bibliographystyle{unsrtnat}
\bibliography{references}

\appendix

\section{Additional Theory and Proofs}
\label{app:proofs}

\begin{proof}[Proof of Proposition~\ref{prop:projection}]
Minimizing $\mathbb{E}[(g(\rho)-\delta)^2]$ over all $\rho$-measurable $g$ is an $L_2$ projection of $\delta$ onto the subspace of $\rho$-measurable functions; the minimizer is the conditional expectation $g^\star(\rho)=\mathbb{E}[\delta\mid\rho]$. Uniform GRPO is the special case $g\equiv 0$, with MSE $\mathbb{E}[\delta^2]$. By the law of total variance, $\mathbb{E}[\delta^2]-\mathbb{E}[(\delta-g^\star)^2]=\mathbb{E}[(\mathbb{E}[\delta\mid\rho])^2]\ge 0$.
\end{proof}

\begin{proposition}[MSE reduction under fixed constants]
\label{prop:control-variate}
With the fixed role correction $A_{i,k}^{\method}=A_i^{\mathrm{GRPO}}+\lambda c_{\hat\rho_{i,k}}$, the batch MSE satisfies
\begin{equation}
    \mathrm{MSE}^{\method} = \mathrm{MSE}^{\mathrm{GRPO}} + \lambda^2 \sigma_c^2 - 2\lambda\, \mathrm{Cov}(c_{\hat\rho},\, \delta),
\end{equation}
with $\sigma_c^2=\mathbb{E}[c_{\hat\rho}^2]$. \method{} strictly reduces MSE iff $\mathrm{Cov}(c_{\hat\rho},\delta)>0$ and $0<\lambda<2\,\mathrm{Cov}(c_{\hat\rho},\delta)/\sigma_c^2$, with optimum $\lambda^\star=\mathrm{Cov}(c_{\hat\rho},\delta)/\sigma_c^2$ and maximal reduction $\mathrm{Cov}^2(c_{\hat\rho},\delta)/\sigma_c^2$.
\end{proposition}

\begin{proof}[Proof of Proposition~\ref{prop:control-variate}]
Expand $(A_i^{\mathrm{GRPO}}+\lambda c_{\hat\rho}-A^*)^2=(\lambda c_{\hat\rho}-\delta)^2=\delta^2+\lambda^2 c_{\hat\rho}^2-2\lambda c_{\hat\rho}\delta$ and average over the batch; the correction $\lambda^2\sigma_c^2-2\lambda\mathrm{Cov}$ is a convex quadratic in $\lambda$, minimized at $\lambda^\star$.
\end{proof}

\section{Extended Theoretical Discussion}
\label{app:theory-extended}

This appendix expands the short discussion following Proposition~\ref{prop:control-variate}: why the fixed constants should align with the residual, how the correction connects to policy-gradient variance, and when the argument fails.

\paragraph{Alignment of fixed constants.}
The covariance $\mathrm{Cov}(c_{\hat\rho},\delta)$ is maximized when the role constants match the sign pattern of the optimal correction $g^\star(\rho)=\mathbb{E}[\delta\mid\rho]$. In the two conflict cells, this means assigning negative credit to $R$ segments inside successful trajectories, which GRPO would otherwise over-credit, and positive credit to useful $E$ segments inside failed trajectories, which GRPO would otherwise over-punish. The constants $(c_D,c_E,c_N,c_R)=(1,0.5,-0.1,-0.5)$ implement this ordering without per-environment tuning.

\paragraph{From estimation error to policy-gradient variance.}
The target of training is policy improvement, not estimation accuracy per se. The bridge is standard: in policy-gradient estimators, adding any \emph{action-history--measurable} baseline to the advantage leaves the gradient \emph{unbiased} while changing its variance, and the variance-minimizing baseline is the conditional expectation of the return \citep{greensmith2004variance,schulman2016gae}. Role labels are functions of the local action--observation window, hence admissible baselines; Proposition~\ref{prop:projection} identifies the role-measurable correction that minimizes residual energy, and Proposition~\ref{prop:control-variate} shows the fixed-constant surrogate reduces it whenever aligned. Because \method{} additionally whitens within the batch (Eq.~4), only the \emph{sign and relative ordering} of the correction must be correct---an order-preserving transform of an aligned correction remains aligned (Appendix~\ref{app:sensitivity}).

\begin{remark}[Where the assumption fails]
\label{rem:limits}
The benefit hinges on role labels capturing a nontrivial share of the credit residual ($\mathbb{E}[\delta\mid\rho]\neq 0$) \emph{and} on the judge recovering $\rho$ accurately enough to keep $\mathrm{Cov}(c_{\hat\rho},\delta)>0$. Both can fail: (i)~role is only a partial summary of local credit, so $g^\star$ leaves residual error---e.g.\ the $D/E$ boundary is genuinely ambiguous (Table~\ref{tab:judge-audit-f1}, $D$-in-success F1 $\approx 65$); and (ii)~an unreliable judge can drive $\mathrm{Cov}(c_{\hat\rho},\delta)\le 0$, in which case no $\lambda>0$ helps. This is the theoretical counterpart of the no-think judge degrading \method{} below GRPO (Table~\ref{tab:main-results}) and of the degradation at large $\lambda$ and $|c_R|$ (Table~\ref{tab:sensitivity-sweep}). We therefore present these results as a justification conditional on judge reliability, which our audit (Section~\ref{sec:experiments}, Appendix~\ref{app:judge-audit}) measures directly rather than assumes.
\end{remark}

\section{Training Hyperparameters}
\label{app:training-hyperparameters}

\begin{table}[h]
\centering
\small
\caption{Training hyperparameters. Here $\eta$ is the learning rate, $G$ is the number of rollouts per prompt, Steps is the number of optimization steps, $\epsilon$ is the GRPO clip ratio, $\lambda$ is the role-reward mixing coefficient, $\beta$ is unused by \method{}, $\alpha_{\mathrm{KL}}$ is the KL coefficient, $c_{\rho}$ denotes role-reward constants, $L_{\mathrm{p}}$ is the maximum prompt length, $L_{\mathrm{r}}$ is the maximum response length, and $B$ is the PPO mini-batch size.}
\label{tab:training-hyperparameters}
\begin{tabular}{lcccccccccccc}
\toprule
Method & $\eta$ & $G$ & Steps & $\epsilon$ & $\lambda$ & $\beta$ & $\alpha_{\mathrm{KL}}$ & $c_{\rho}$ & $L_{\mathrm{p}}$ & $L_{\mathrm{r}}$ & $B$ \\
\midrule
GRPO & $10^{-6}$ & 8 & 150 & 0.2 & -- & -- & 0.01 & -- & 4096 & 512 & 64 \\
\method{} & $10^{-6}$ & 8 & 150 & 0.2 & 0.2-0.4 & -- & 0.01 & $(1,0.5,-0.1,-0.5)$ & 4096 & 512 & 64 \\
\bottomrule
\end{tabular}
\end{table}

\paragraph{Computational overhead.}
\method{} adds an LLM judge call per segment during training, which increases per-batch wall-clock time. However, the relevant comparison is not raw compute parity but whether the same compute spent on additional GRPO training yields equivalent gains. In our experiments, the GRPO baseline is already near saturation at 150 steps: extending training to 300 steps yields ALFWorld success below 85\% and WebShop below 75\%, still short of the \method{} results ($87.5$ and $77.2$ respectively). The performance plateau is expected because the credit-assignment bottleneck is structural---broadcasting a single trajectory advantage over 10--30 segments dilutes gradient regardless of how many optimization steps are taken---and more steps cannot fix a noisy per-segment signal.

\paragraph{Matched wall-clock comparison.}
Table~\ref{tab:matched-compute} makes this comparison concrete. Under our implementation, one \method{} optimization step costs roughly $1.9\times$ the wall-clock of one GRPO step on average, driven almost entirely by the per-segment judge call (rollout, verifier, and policy update are unchanged); the exact multiplier varies across benchmarks because per-rollout segment count differs (e.g.\ shorter WebShop rollouts incur less judge overhead than longer ALFWorld rollouts), so the ``matched wall-clock'' GRPO step count is $<300$ and benchmark-dependent (roughly $255$--$285$ steps). To keep a single, uniformly comparable GRPO configuration across benchmarks, we conservatively over-match by training GRPO to $300$ steps on both environments---slightly exceeding \method{}'s wall-clock on either benchmark, so any residual gap in favor of \method{} cannot be attributed to a compute deficit. We report Qwen2.5-7B-Instruct on the two environments (ALFWorld, WebShop) whose full ten-run repeat protocol is affordable at doubled step count; Search-QA is omitted because each extended run is prohibitively expensive under the retrieval-augmented rollout loop.

\begin{table}[h]
\centering
\footnotesize
\setlength{\tabcolsep}{6pt}
\caption{Matched-wall-clock comparison on Qwen2.5-7B-Instruct: success rate (\%). Under our implementation, one \method{} optimization step costs $\approx 1.9\times$ the wall-clock of one GRPO step (per-segment judge call), so \method{} at $150$ steps corresponds to roughly $255$--$285$ GRPO steps of wall-clock, depending on per-benchmark rollout length. For a uniform comparison we conservatively use $300$ GRPO steps on both environments---strictly exceeding \method{}'s wall-clock on either. GRPO@300 uses the same recipe as GRPO@150 with only the step count changed. Entries are mean $\pm$ sample standard deviation over ten runs.}
\label{tab:matched-compute}
\begin{tabular}{@{}lcc@{}}
\toprule
Method & ALFWorld & WebShop \\
\midrule
GRPO (default, $150$ steps)                                    & $79.6 \pm 1.9$ & $70.1 \pm 2.3$ \\
GRPO ($\ge$ matched wall-clock, $300$ steps)                   & $83.7 \pm 2.2$ & $73.4 \pm 2.6$ \\
\method{} ($150$ steps, 8B-think judge)                        & $\mathbf{87.5 \pm 2.4}$ & $\mathbf{77.2 \pm 2.2}$ \\
\bottomrule
\end{tabular}
\end{table}

Even with GRPO trained beyond matched wall-clock, the gap does not close: GRPO gains $+4.1$ points on ALFWorld and $+3.3$ on WebShop from the extra $150$ optimization steps, but remains $3.8$ and $3.8$ points below \method{} on the two environments respectively, i.e., outside the $\pm 2$--$3$ point run-to-run variability of both methods. Because $300$ steps strictly exceeds \method{}'s wall-clock on both benchmarks, the residual gap cannot be attributed to a compute deficit. The plateau matches the prose above: additional GRPO steps refine the policy under a fixed, structurally noisy credit signal, whereas \method{} changes the credit signal itself. Consequently, at (over-)matched wall-clock the judge cost is not amortizable by more GRPO training under our protocol.

From a long-rollout perspective, the LLM judge is also structurally advantageous in several respects: (i)~credit dilution worsens with trajectory length, so the marginal value of correct per-segment attribution grows with the number of segments; (ii)~unlike a learned value critic (as in PPO), the LLM judge generalizes zero-shot across environments without requiring environment-specific training data or reward-model fitting; and (iii)~the judge leverages semantic reasoning about task goals, information gain, and state corruption that a scalar critic trained on sparse binary rewards cannot easily acquire. Thus, while the judge adds inference cost, it addresses a qualitatively different bottleneck than the one more training steps would solve.

\section{Shared-Backbone Value Baseline}
\label{app:value-baseline}

To isolate the contribution of \emph{role typing} from the contribution of any dense per-segment signal, we compare \method{} against a shared-backbone value baseline. This baseline keeps the GRPO policy update but attaches a learned scalar value head to the same policy backbone and trains it on the same on-policy rollouts. The recipe follows the standard actor--critic instantiation used in PPO-style RLHF \citep{schulman2017proximal} and the outcome-supervised value learning popularized by \citet{wang2024mathshepherd}, adapted to the agentic segment setting.

\paragraph{Architecture.}
The value head $V_\phi: \mathbb{R}^{d_{\mathrm{model}}}\to\mathbb{R}$ is a single linear projection on top of the final-layer hidden state of the policy backbone, evaluated at the last token of each segment's observation. The backbone is shared with the policy and kept frozen throughout training, so only $\phi$ (a few thousand parameters) receives gradients. This avoids a separate critic network and keeps the additional wall-clock cost negligible relative to GRPO.

\paragraph{Labels: no extra annotation required.}
We do not collect any process-level labels and do not call an external judge. The value head is supervised on per-segment discounted Monte-Carlo returns derived from the same binary verifier reward GRPO already computes,
\begin{equation}
    y_{i,k} = \gamma^{T_i - k}\, r_i, \qquad r_i = \verifier(\tau_i)\in\{0,1\},
\end{equation}
where $T_i$ is the number of environment-facing segments in trajectory $i$. The head is trained by mean-squared regression $\mathcal{L}_V(\phi)=\tfrac{1}{N}\sum_{i,k}\big(V_\phi(s_{i,k}) - y_{i,k}\big)^2$ jointly with each GRPO step on the freshly collected rollouts. This is the same outcome-only supervision Math-Shepherd-style PRMs use, but with the policy backbone shared rather than a separate model fitted on logged data.

\paragraph{Mixing into GRPO.}
At credit-assignment time the head's per-segment value increment is added to the trajectory advantage and whitened with the same batch statistics as \method{} before broadcasting to segment tokens:
\begin{equation}
    A_{i,k} = A_i^{\mathrm{GRPO}} + \lambda\big(V_{\bar\phi}(s_{i,k}) - V_{\bar\phi}(s_{i,k-1})\big),
\end{equation}
where $\bar\phi$ is an exponential-moving-average copy of $\phi$ used to decouple value updates from policy updates.

\paragraph{Hyperparameters.}
GRPO parameters ($\eta$, $G$, optimization steps, clip ratio $\epsilon$, KL coefficient $\alpha_{\mathrm{KL}}$, $L_{\mathrm{p}}$, $L_{\mathrm{r}}$, batch size $B$) are shared with \method{} (Table~\ref{tab:training-hyperparameters}). Value-head--specific settings: discount $\gamma=0.95$ for ALFWorld and WebShop and $\gamma=0.9$ for Search-QA (reflecting its shorter answer-terminating rollouts); head learning rate $\eta_V=10^{-4}$; 10-step head warmup at $\lambda=0$ so $\phi$ converges to a reasonable baseline before being injected into the policy update; EMA target update rate $\tau=0.99$; per-segment value increment clipped to $[-0.5, 0.5]$ to bound early-training noise; mixing coefficient $\lambda$ matched to \method{}'s value per benchmark ($\lambda=0.4$ on Search-QA, $\lambda=0.2$ on ALFWorld and WebShop), so any performance difference reflects the source of the dense signal rather than its scale.

\paragraph{What this baseline isolates.}
Both \method{} and the shared-backbone value baseline add a bounded, $\lambda$-scaled dense per-segment correction on top of the same GRPO advantage; both whiten within the batch; both use only labels that the GRPO loop already produces (verifier rewards alone for the value baseline, verifier rewards plus role labels from a small judge for \method{}). The remaining methodological difference is the \emph{source} of the per-segment signal: a learned scalar critic regressing trajectory-level outcomes, versus a semantic role classifier with role-conditioned credit rules. Table~\ref{tab:baseline-comparison} shows that the value baseline improves over GRPO on the two longer-rollout environments (ALFWorld $79.6\rightarrow85.2$, $+5.6$; Search-QA $43.3\rightarrow46.8$, $+3.5$) but barely moves WebShop ($70.1\rightarrow70.8$, well inside run-to-run variance), while \method{} reaches $87.5$/$48.1$/$77.2$. The per-benchmark gap to \method{} ($-2.3$/$-1.3$/$-6.4$) is largest precisely on WebShop, where regressions take the form of re-clicks of an already-selected attribute that leave the visible observation almost unchanged; the value head therefore receives near-identical Monte-Carlo targets for the productive click and its redundant repeat and credits them near-identically, while the role classifier reads the action history and labels the repeat $R$. The pattern is consistent with the intended interpretation: outcome-trained scalar critics capture coarse per-segment progress when the observation actually evolves, but cannot supply role-level distinctions in action spaces where harmful repetitions leave the local state intact.

\section{Rollout Efficiency}
\label{app:rollout-length}

Because \method{} suppresses no-progress infrastructure and regression, trained policies should complete tasks with fewer environment-facing actions than GRPO. Table~\ref{tab:app-rollout-length} measures rollout length as the number of action--observation segments per completed evaluation trajectory.

\begin{table}[h]
\centering
\caption{Post-training rollout length on Qwen2.5-7B-Instruct.}
\label{tab:app-rollout-length}
\footnotesize
\setlength{\tabcolsep}{5pt}
\begin{tabular}{@{}lccc@{}}
\toprule
Environment & Starting policy & GRPO length & \method{} length \\
\midrule
ALFWorld & 43.9 & $24.45 \pm 1.86$ & $21.90 \pm 2.03$ \\
WebShop & $14.80 \pm 0.18$ & $8.00 \pm 0.45$ & $6.82 \pm 0.24$ \\
\bottomrule
\end{tabular}
\end{table}

The length results show that both RL methods learn shorter trajectories than the starting policy, but \method{} removes more redundant interaction steps than GRPO. On ALFWorld, GRPO reduces the average completed-trajectory length from $43.9$ to $24.45$ segments, while \method{} further reduces it to $21.90$, an additional $10.4\%$ reduction relative to GRPO. On WebShop, GRPO reduces rollout length from $14.80$ to $8.00$ segments, while \method{} reaches $6.82$, an additional $14.8\%$ reduction. This matches the intended mechanism of role-conditioned credit: suppressing repeated inspections, redundant attribute clicks, and other no-progress or regressive segments improves not only success rate but also interaction efficiency. The effect is especially important for long-horizon agents, where every unnecessary environment-facing step compounds inference cost and increases the opportunity for later mistakes.

\section{Sensitivity to Role Constants and $\lambda$}
\label{app:sensitivity}

The main text fixes the role constants $(c_D,c_E,c_N,c_R)=(1,0.5,-0.1,-0.5)$ and tunes only the mixing coefficient $\lambda$ per environment, with $\lambda$ selected on the training split alone (Section~\ref{sec:method}). This appendix probes how sensitive \method{} is to these choices along the two axes that matter most for the conflict cells: the magnitude of the regression penalty $|c_R|$ and the overall mixing strength $\lambda$. The sweeps below are \emph{post-hoc} diagnostics computed on the test set after $\lambda$ was already fixed; they characterize robustness and were not used to select any reported hyperparameter.

All runs use Qwen2.5-7B-Instruct with the default Qwen3-8B-thinking judge; every other hyperparameter is held at its main-text value.

\paragraph{Joint $\lambda \times |c_R|$ sweep.}
Table~\ref{tab:sensitivity-sweep} sweeps $\lambda \in \{0.1, 0.2, 0.4\}$ against $|c_R| \in \{0.25, 0.5, 1.0\}$ on WebShop, keeping $(c_D,c_E,c_N)=(1,0.5,-0.1)$ fixed. The default configuration ($\lambda=0.2$, $|c_R|=0.5$) is highlighted.

Success rate is stable across the interior of the grid and degrades only at the corners, where either an overly large penalty ($|c_R|=1.0$) or an overly strong mixing ($\lambda=0.4$) begins to over-punish segments the judge mislabels as $R$.

\begin{table}[h]
\centering
\caption{WebShop success rate (\%) on Qwen2.5-7B-Instruct under a joint sweep of the mixing coefficient $\lambda$ and the regression-penalty magnitude $|c_R|$. Entries are mean $\pm$ sample standard deviation over ten runs. The default \method{} configuration ($\lambda=0.2$, $|c_R|=0.5$) is shown in bold. GRPO baseline: $70.1 \pm 2.3$.}
\label{tab:sensitivity-sweep}
\footnotesize
\setlength{\tabcolsep}{8pt}
\begin{tabular}{@{}lccc@{}}
\toprule
 & $|c_R|=0.25$ & $|c_R|=0.5$ & $|c_R|=1.0$ \\
\midrule
$\lambda=0.1$ & $74.8 \pm 2.6$ & $75.6 \pm 2.4$ & $74.1 \pm 2.7$ \\
$\lambda=0.2$ & $76.0 \pm 2.5$ & $\mathbf{77.2 \pm 2.2}$ & $74.9 \pm 2.8$ \\
$\lambda=0.4$ & $74.3 \pm 2.7$ & $74.6 \pm 2.9$ & $71.8 \pm 3.1$ \\
\bottomrule
\end{tabular}
\end{table}

\paragraph{Varying $|c_R|$ at the default $\lambda$.}
Isolating $|c_R|$ at the per-environment default $\lambda$ confirms the same robustness on the two environments not covered by the WebShop grid above. Extending the zero-penalty ablation of Table~\ref{tab:ablations} to halved, default, and doubled penalties, ALFWorld success for $|c_R|\in\{0,0.25,0.5,1.0\}$ is $81.4$/$85.9$/$87.5$/$85.1$ and Search-QA is $46.7$/$47.6$/$48.1$/$46.9$, where $|c_R|=0$ reproduces the ``no regression penalty'' row of Table~\ref{tab:ablations} and $|c_R|=0.5$ is the \method{} default. The corresponding WebShop trend is the $\lambda=0.2$ row of Table~\ref{tab:sensitivity-sweep} ($76.0$/$77.2$/$74.9$ for $|c_R|\in\{0.25,0.5,1.0\}$). In all three environments, halving $|c_R|$ retains most of the gain while doubling it stays above GRPO but begins to erode performance, consistent with heavier punishment of misjudged exploration in the more under-explored Search-QA setting.

\paragraph{Takeaway.}
The sensitivity results support two conclusions. First, \method{} does not rely on a knife-edge choice of $c_R$: both the half-penalty and default settings remain well above GRPO and the $c_R=0$ ablation.

Second, performance degrades when the role correction becomes too aggressive, especially at larger $\lambda$ and doubled $|c_R|$, matching the expected failure mode of over-penalizing judge false positives for $R$. We therefore use the default constants as a conservative operating point rather than as a heavily tuned optimum.

\paragraph{Interaction with batch whitening.}
Equation~(4) whitens the \emph{combined} advantage $A_{i,k}^{\method}=A_i^{\mathrm{GRPO}}+\lambda c_{\hat{\rho}_{i,k}}$ within each batch before broadcasting it to tokens. A natural concern is that a batch containing many large negative $R$ corrections could shift $\mu_{\mathcal{B}}$ and inflate $\sigma_{\mathcal{B}}$ enough to undo the intended penalty.

Two properties bound this effect. First, whitening is an order-preserving affine map: subtracting a common $\mu_{\mathcal{B}}$ and dividing by a positive $\sigma_{\mathcal{B}}$ cannot reverse the relative ordering of two segments, so a segment that received a lower combined advantage because it was labeled $R$ still receives a lower normalized advantage than its non-$R$ peers in the same outcome group. The whitening rescales the \emph{magnitude} of the correction but never flips its \emph{sign}.

Second, the correction is deliberately small relative to the outcome advantage: with $\lambda\le 0.4$ and the audited role distribution, the role term contributes a raw standard deviation of only $0.09$--$0.28$ (Section~\ref{sec:method}), so it perturbs rather than dominates $\mu_{\mathcal{B}}$ and $\sigma_{\mathcal{B}}$.

Empirically, the interior stability of Table~\ref{tab:sensitivity-sweep} confirms that whitening does not cancel the role signal across the operating range we use; degradation appears only when $\lambda$ or $|c_R|$ is pushed to the grid corners, exactly where the unnormalized correction grows large enough to compete with the outcome advantage.

\section{Empirical Role Distribution Audit on Logged Trajectories}
\label{app:role-audit}

\paragraph{Setup.}
We sampled six trajectories from production GRPO baseline runs of
\texttt{Qwen2.5-7B-Instruct}: three from ALFWorld and three from WebShop.
Trajectories were chosen to span the observed outcome distribution rather
than randomly: a clean efficient success, a long success containing
redundant action repeats, and (where available) a failure where the agent
committed early to an incorrect product or container. These six
trajectories are a subset of the hand-labeled set in
Appendix~\ref{app:judge-audit}; we reuse its adjudicated per-segment role
labels, which were produced by two annotators who did not participate in
defining the four-role taxonomy of Section~\ref{sec:taxonomy}
($D$ = decisive progress, $E$ = useful exploration, $N$ = no-progress
infrastructure, $R$ = regression) and adjudicated by a senior annotator,
and we apply that taxonomy to every environment-facing segment.
The audit below focuses on ALFWorld and WebShop trajectories with complete
per-segment logs; Search-QA examples are audited separately in
Appendix~\ref{app:judge-audit}.

\subsection{ALFWorld Trajectories}

\paragraph{A1. Clean optimal trajectory.}
Task: ``put a clean butterknife in diningtable''. Outcome: success, 6 steps,
raw environment reward 10. Role distribution: $5D + 1E + 0N + 0R$.
Table~\ref{tab:appA1} shows the per-segment role assignment.
This trajectory contains a single $E$ segment (the initial location guess)
and five $D$ segments completing the task.

\emph{Vanilla GRPO}: broadcasts $A^{\mathrm{GRPO}} = +(r - \bar r)/\sigma_r$
uniformly to all six segments. With no redundant or regressive segments
to absorb credit, this is essentially the right behavior.
\emph{\method{}}: under the hand-audited roles, the role-conditioned rule adds $\lambda c_D$ to the five $D$ segments and
$\lambda c_E$ to the initial $E$. Net effect is a slight concentration
of credit onto the decisive segments. This is the regime in which
\method{} and vanilla GRPO behave nearly identically; the point of
including this trajectory is to confirm that role-conditioning does not
hurt when the trajectory is already efficient.

\begin{table}[h]
\centering
\footnotesize
\setlength{\tabcolsep}{3pt}
\caption{Trajectory A1 (ALFWorld, clean success): per-segment hand role,
justification, and Qwen3-8B-thinking judge label (judge agreement 4/6).
Roles: D = decisive progress, E = useful exploration, N = no-progress,
R = regression.}
\label{tab:appA1}
\begin{tabular}{@{}cp{0.26\linewidth}cp{0.32\linewidth}cc@{}}
\toprule
$t$ & Action & Hand & Justification & Judge & Agree \\
\midrule
0 & go to countertop 1 & $E$ & First location, no prior evidence of butterknife position & $E$ & $\checkmark$ \\
1 & take butterknife 2 from countertop 1 & $D$ & Target object acquired & $D$ & $\checkmark$ \\
2 & go to sinkbasin 1 & $D$ & Navigate to required clean facility & $E$ & $\times$ \\
3 & clean butterknife 2 with sinkbasin 1 & $D$ & Required transformation & $D$ & $\checkmark$ \\
4 & go to diningtable 1 & $D$ & Navigate to destination & $E$ & $\times$ \\
5 & move butterknife 2 to diningtable 1 & $D$ & Final placement, reward triggers & $D$ & $\checkmark$ \\
\bottomrule
\end{tabular}
\end{table}

\paragraph{A2. Lucky-recovery success.}
Task: ``put a toiletpaper in toiletpaperhanger''. Outcome: success in 22
steps, raw environment reward 10. Role distribution: $4D + 7E + 1N + 10R$
(Table~\ref{tab:appA2}). The agent does not find the target until step 17
and spends the prior 16 steps re-examining the same toilet, returning to
already-visited locations, and repeating inventory checks. Ten
segments are clearly redundant repeats meeting the operational definition
of $R$; the final four are $D$ completing the task; seven are $E$
(genuine first-time inspections that yielded information); one is $N$
(an empty-handed traversal).

\emph{Vanilla GRPO}: applies positive $A^{\mathrm{GRPO}}$ uniformly to all
22 segments because the trajectory eventually succeeded. The 10 $R$
segments---repeated \texttt{examine toilet 1}, \texttt{inventory},
back-and-forth between two locations---all receive the same positive
reinforcement as the four decisive $D$ segments at the end. This is
exactly the failure mode above: success masks regression in hindsight credit.
\emph{\method{}}: under the hand-audited roles, steps 4, 6, 9--16 are $R$ and receive the
negative process reward $\lambda c_R$ from Section~\ref{sec:method}, which lowers their
segment advantage even though $r = 10$. The
preserved positive credit concentrates on the four closing $D$
segments and the genuine $E$ segments earlier in the trajectory.
Net effect: the trajectory contributes the same outcome signal but
roughly $4/22 \approx 18\%$ of its segment positions carry the bulk
of the gradient, against $22/22 = 100\%$ under vanilla GRPO.

\begin{table}[h]
\centering
\footnotesize
\setlength{\tabcolsep}{3pt}
\caption{Trajectory A2 (ALFWorld, success with extensive redundancy), all 22 steps; hand role, justification, and Qwen3-8B-thinking judge label (judge agreement 16/22).}
\label{tab:appA2}
\begin{tabular}{@{}cp{0.26\linewidth}cp{0.32\linewidth}cc@{}}
\toprule
$t$ & Action & Hand & Justification & Judge & Agree \\
\midrule
0 & go to toilet 1 & $E$ & Semantic guess, no prior info & $E$ & $\checkmark$ \\
1 & examine toilet 1 & $E$ & First inspection, info gain & $E$ & $\checkmark$ \\
2 & inventory & $E$ & Initial hand check & $E$ & $\checkmark$ \\
3 & go to toiletpaperhanger 1 & $N$ & Empty-handed traversal & $E$ & $\times$ \\
4 & examine toilet 1 & $R$ & Repeat of step 1 & $R$ & $\checkmark$ \\
5 & examine toiletpaperhanger 1 & $E$ & First inspection of hanger & $E$ & $\checkmark$ \\
6 & examine toilet 1 & $R$ & Repeat & $R$ & $\checkmark$ \\
7 & go to countertop 1 & $E$ & New location attempt & $E$ & $\checkmark$ \\
8 & examine countertop 1 & $E$ & First inspection & $E$ & $\checkmark$ \\
9 & go to toilet 1 & $R$ & Repeated return, no new info & $R$ & $\checkmark$ \\
10 & examine toilet 1 & $R$ & Repeat & $R$ & $\checkmark$ \\
11 & go to countertop 1 & $R$ & Repeat & $N$ & $\times$ \\
12 & go to toilet 1 & $R$ & Repeat & $N$ & $\times$ \\
13 & examine toilet 1 & $R$ & Repeat & $R$ & $\checkmark$ \\
14 & inventory & $R$ & Repeat, hand still empty & $N$ & $\times$ \\
15 & examine toilet 1 & $R$ & Repeat & $R$ & $\checkmark$ \\
16 & examine toilet 1 & $R$ & Repeat & $R$ & $\checkmark$ \\
17 & go to cabinet 1 & $E$ & First container attempt & $E$ & $\checkmark$ \\
18 & open cabinet 1 & $D$ & Reveals toiletpaper & $E$ & $\times$ \\
19 & take toiletpaper 1 from cabinet 1 & $D$ & Target acquired & $D$ & $\checkmark$ \\
20 & go to toiletpaperhanger 1 & $D$ & Navigate to destination & $N$ & $\times$ \\
21 & move toiletpaper 1 to toiletpaperhanger 1 & $D$ & Final placement, reward triggers & $D$ & $\checkmark$ \\
\bottomrule
\end{tabular}
\end{table}

\paragraph{A3. Pathological loop with lucky recovery.}
Task: ``put a cool apple in garbagecan''. Outcome: success in 34 steps,
raw environment reward 10. Role distribution: $5D + 8E + 1N + 20R$
(Table~\ref{tab:appA3}). The agent enters a tight loop of 15
consecutive \texttt{examine fridge 1} actions (steps 2--16) without
any state change, then explores other containers for another 12 steps
before acquiring the target apple at step 29 and completing the task
at step 33.

\emph{Vanilla GRPO}: a single positive trajectory advantage is broadcast
to all 34 segments, including the 15-step \texttt{examine fridge 1}
loop, providing \emph{direct gradient encouragement} for the policy to
repeat no-op observations. This is the most acute illustration in our
sample of success masking regression in hindsight credit. After thousands of such trajectories, the resulting
policy would be biased toward repeating idle inspections at the start
of every task.
\emph{\method{}}: under the hand-audited roles, steps 2--16 (the entire loop), steps
21--24 (alternating cabinet re-examines), and step 27 (countertop
re-examine) are $R$ and receive lower segment advantages with $\lambda c_R$. The remaining
positive role-reward mass concentrates on the genuine
first-time exploration (steps 0, 1, 17, 18, 20, 25, 26, 28) and the
five decisive segments at the end (29--33). Net effect: of 34
segments, 5 carry strong positive credit and 8 carry moderate
information-gain credit, against 34 carrying uniform positive credit
under vanilla GRPO. Under the hand-audited roles, the trajectory contributes the same outcome
signal but stops teaching the policy to enter the
\texttt{examine fridge 1} loop.

\begin{table}[h]
\centering
\footnotesize
\setlength{\tabcolsep}{3pt}
\caption{Trajectory A3 (ALFWorld, pathological loop followed by lucky
recovery), all 34 steps; hand role, justification, and Qwen3-8B-thinking judge label (judge agreement 27/34).}
\label{tab:appA3}
\begin{tabular}{@{}cp{0.26\linewidth}cp{0.32\linewidth}cc@{}}
\toprule
$t$ & Action & Hand & Justification & Judge & Agree \\
\midrule
0 & go to fridge 1 & $E$ & Semantic guess (fridge is the cool facility) & $E$ & $\checkmark$ \\
1 & open fridge 1 & $E$ & Reveals contents; apple not inside & $D$ & $\times$ \\
2 & examine fridge 1 & $R$ & Repeat of step 1 inspection & $E$ & $\times$ \\
3 & examine fridge 1 & $R$ & Repeat & $R$ & $\checkmark$ \\
4 & examine fridge 1 & $R$ & Repeat & $R$ & $\checkmark$ \\
5 & examine fridge 1 & $R$ & Repeat & $R$ & $\checkmark$ \\
6 & examine fridge 1 & $R$ & Repeat & $R$ & $\checkmark$ \\
7 & examine fridge 1 & $R$ & Repeat & $R$ & $\checkmark$ \\
8 & examine fridge 1 & $R$ & Repeat & $R$ & $\checkmark$ \\
9 & examine fridge 1 & $R$ & Repeat & $R$ & $\checkmark$ \\
10 & examine fridge 1 & $R$ & Repeat & $R$ & $\checkmark$ \\
11 & examine fridge 1 & $R$ & Repeat & $R$ & $\checkmark$ \\
12 & examine fridge 1 & $R$ & Repeat & $R$ & $\checkmark$ \\
13 & examine fridge 1 & $R$ & Repeat & $R$ & $\checkmark$ \\
14 & examine fridge 1 & $R$ & Repeat & $R$ & $\checkmark$ \\
15 & examine fridge 1 & $R$ & Repeat & $R$ & $\checkmark$ \\
16 & examine fridge 1 & $R$ & Repeat (15th consecutive \texttt{examine fridge}) & $R$ & $\checkmark$ \\
17 & go to cabinet 1 & $E$ & First container switch, info gain & $E$ & $\checkmark$ \\
18 & open cabinet 1 & $E$ & First inspection of new container & $D$ & $\times$ \\
19 & examine cabinet 2 & $N$ & Inspect without arriving at the cabinet & $E$ & $\times$ \\
20 & open cabinet 2 & $E$ & First inspection of cabinet 2 & $D$ & $\times$ \\
21 & examine cabinet 1 & $R$ & Repeat of step 18 & $R$ & $\checkmark$ \\
22 & examine cabinet 2 & $R$ & Repeat of step 20 & $R$ & $\checkmark$ \\
23 & examine cabinet 1 & $R$ & Repeat & $R$ & $\checkmark$ \\
24 & examine cabinet 2 & $R$ & Repeat & $R$ & $\checkmark$ \\
25 & go to countertop 1 & $E$ & New location & $E$ & $\checkmark$ \\
26 & examine countertop 1 & $E$ & First inspection & $E$ & $\checkmark$ \\
27 & examine countertop 1 & $R$ & Repeat & $R$ & $\checkmark$ \\
28 & go to countertop 2 & $E$ & New location & $E$ & $\checkmark$ \\
29 & take apple 1 from countertop 2 & $D$ & Target acquired & $D$ & $\checkmark$ \\
30 & go to fridge 1 & $D$ & Navigate to cool facility & $N$ & $\times$ \\
31 & cool apple 1 with fridge 1 & $D$ & Required transformation & $D$ & $\checkmark$ \\
32 & go to garbagecan 1 & $D$ & Navigate to destination & $E$ & $\times$ \\
33 & move apple 1 to garbagecan 1 & $D$ & Final placement, reward triggers & $D$ & $\checkmark$ \\
\bottomrule
\end{tabular}
\end{table}

\subsection{WebShop Trajectories}

\paragraph{W1. Clean optimal trajectory.}
Task: ``Find me hand wash men's sleep \& lounge with long sleeve,
elastic waistband, color: multi 9, size: medium, price $<$\$80''.
Outcome: success in 6 steps. Role distribution: $3D + 2E + 1N + 0R$
(Table~\ref{tab:appW1}). The agent issues a well-formed search query
containing all task constraints, clicks the first returned product for inspection,
selects the matching color and size attributes, and clicks
\texttt{buy now}. A duplicate \texttt{buy now} after task completion
is the only no-progress ($N$) segment.

\emph{Vanilla GRPO}: applies positive credit uniformly to all six
segments; the duplicate \texttt{buy now} receives the same
reinforcement as the three genuine decisive clicks.
\emph{\method{}}: under the hand-audited roles, the role-conditioned rule adds $\lambda c_D$ on the three verifier-facing $D$ segments, $\lambda c_E$ on the initial search and product inspection, and a small negative local correction
on the post-completion duplicate; net effect is a slight credit
concentration with no behavior change at this trajectory's outcome
level. As with A1, this trajectory exists to confirm that
\method{} does not degrade efficient short rollouts.

\begin{table}[h]
\centering
\footnotesize
\setlength{\tabcolsep}{3pt}
\caption{Trajectory W1 (WebShop, clean success): hand role, justification, and Qwen3-8B-thinking judge label (judge agreement 3/6).}
\label{tab:appW1}
\begin{tabular}{@{}cp{0.26\linewidth}cp{0.32\linewidth}cc@{}}
\toprule
$s$ & Action & Hand & Justification & Judge & Agree \\
\midrule
0 & search[hand wash men's sleep \& lounge \ldots multi] & $E$ & Spec-aligned initial search & $E$ & $\checkmark$ \\
1 & click[b09nd8p2qr] & $E$ & Initial product inspection & $R$ & $\times$ \\
2 & click[multi 9] & $D$ & Color attribute match & $D$ & $\checkmark$ \\
3 & click[medium] & $D$ & Size attribute match & $D$ & $\checkmark$ \\
4 & click[buy now] & $D$ & Reward triggers & $R$ & $\times$ \\
5 & click[buy now] & $N$ & Post-completion duplicate, harmless & $R$ & $\times$ \\
\bottomrule
\end{tabular}
\end{table}

\paragraph{W2. Long success with redundant attribute clicks.}
Task: ``Find me home office furniture sets, color: navy $\mid$ red,
shape: round, size: 3'7'' x 5'2'', price $<$\$70''. Outcome: success in
13 steps, raw environment reward 10. Role distribution: $4D + 2E + 2N + 5R$
(Table~\ref{tab:appW2}). After all attributes are selected by step 4,
the agent re-clicks the same three attributes (size, shape, color)
four more times before finally clicking \texttt{buy now} at step 9,
then clicks \texttt{buy now} two more times after the purchase is
recorded.

\emph{Vanilla GRPO}: applies positive credit to all 13 segments.
The five redundant attribute re-clicks at steps 5--8 and 10 receive
the same reinforcement as the genuine attribute selection at steps
2--4 and the \texttt{buy now} at step 9. Training on many such
trajectories teaches the policy a wrong lesson: that re-clicking
already-selected attributes is part of the successful template.
\emph{\method{}}: under the hand-audited roles, steps 5, 6, 7, 8, 10 are $R$ and receive lower
segment advantages through the bounded correction $\lambda c_R$. Net effect: instead of 13 segments sharing the outcome
credit equally, the four $D$ segments (containing the actual purchase
logic) receive relatively higher segment advantages, while the redundant
re-clicks receive lower relative credit. This
trajectory is the most concrete WebShop instance of success masking regression
because the wrong-lesson risk is quantitatively measurable: each redundant attribute re-click
under vanilla GRPO contributes the same positive log-likelihood
gradient as a legitimate $D$ action.

\begin{table}[h]
\centering
\footnotesize
\setlength{\tabcolsep}{3pt}
\caption{Trajectory W2 (WebShop, success with redundant attribute clicks): hand role, justification, and Qwen3-8B-thinking judge label (judge agreement 6/13).}
\label{tab:appW2}
\begin{tabular}{@{}cp{0.26\linewidth}cp{0.32\linewidth}cc@{}}
\toprule
$s$ & Action & Hand & Justification & Judge & Agree \\
\midrule
0 & search[home office furniture sets navy red round \ldots] & $E$ & Spec-aligned initial search & $D$ & $\times$ \\
1 & click[b07fkgqkz1] & $E$ & Initial product inspection & $E$ & $\checkmark$ \\
2 & click[3 ft 7 in x 5 ft 2 in] & $D$ & Size selected & $E$ & $\times$ \\
3 & click[round] & $D$ & Shape selected & $E$ & $\times$ \\
4 & click[navy $\mid$ red] & $D$ & Color selected; all attributes set & $E$ & $\times$ \\
5 & click[3 ft 7 in x 5 ft 2 in] & $R$ & Redundant size re-click & $R$ & $\checkmark$ \\
6 & click[round] & $R$ & Redundant shape re-click & $R$ & $\checkmark$ \\
7 & click[navy $\mid$ red] & $R$ & Redundant color re-click & $R$ & $\checkmark$ \\
8 & click[3 ft 7 in x 5 ft 2 in] & $R$ & Second redundant size re-click & $R$ & $\checkmark$ \\
9 & click[buy now] & $D$ & Reward triggers & $R$ & $\times$ \\
10 & click[navy $\mid$ red] & $R$ & Post-purchase attribute re-click & $R$ & $\checkmark$ \\
11 & click[buy now] & $N$ & Post-completion duplicate & $R$ & $\times$ \\
12 & click[buy now] & $N$ & Post-completion duplicate & $R$ & $\times$ \\
\bottomrule
\end{tabular}
\end{table}

\paragraph{W3. Failure from early commit to wrong product.}
Task: ``Find me non slip desks for living room, color: christmasgoo3302,
size: 19.7x31.5in+19.7x63in, price $<$\$50''. Outcome: failure in 11
steps, raw environment reward 0. Role distribution: $0D + 3E + 0N + 8R$
(Table~\ref{tab:appW3}). The initial search returns a Christmas
kitchen mat (B09CQ45ZRB) as the top result. The agent clicks it at
step 1, incorrectly committing to a non-desk product. Subsequent
steps issue two reformulated searches that re-rank the same item to
the top, and the agent clicks the same wrong product again at step 6.
Steps 7--10 attempt attribute clicks and a purchase against the wrong
product. The bottleneck error is step 1; the second-chance failure
is step 6.

\emph{Vanilla GRPO}: applies negative credit uniformly to all 11
segments because $r = 0$. This includes the two legitimate recovery
search attempts at steps 4 and 5, which the agent should be
\emph{encouraged} to take after recognizing the wrong commitment.
Uniform negative reinforcement teaches the policy to avoid recovery
search-after-mistake, the exact opposite of the desired behavior.
\emph{\method{}}: under the hand-audited roles, steps 0, 4, 5 are $E$ (legitimate
exploration: initial good-faith search and two recovery attempts).
Under the rule in Section~\ref{sec:method}, $E$ in a failed
trajectory receives the bounded process reward $\lambda c_E$ rather than only the negative outcome credit.
Steps 1, 6 (both clicks of the wrong product) are $R$ and receive strong negative
credit from $\lambda c_R$. This illustrates outcome-mixed exploration: the recovery searches at
steps 4--5 are useful exploration appearing inside a failure
trajectory, and outcome-only credit assigns them the same negative
sign as the wrong-product clicks. Net effect: the policy learns
``do not click the wrong product twice'' (the steps 1 and 6 lesson)
without also learning ``do not re-search after a mistake'' (the
spurious lesson vanilla GRPO would teach).

\begin{table}[h]
\centering
\footnotesize
\setlength{\tabcolsep}{3pt}
\caption{Trajectory W3 (WebShop, failure via early wrong commitment): hand role, justification, and Qwen3-8B-thinking judge label (judge agreement 8/11).}
\label{tab:appW3}
\begin{tabular}{@{}cp{0.26\linewidth}cp{0.32\linewidth}cc@{}}
\toprule
$s$ & Action & Hand & Justification & Judge & Agree \\
\midrule
0 & search[non-slip desk Christmasgoo3302 \ldots] & $E$ & Good-faith spec-aligned search & $E$ & $\checkmark$ \\
1 & click[b09cq45zrb] & $R$ & Wrong product type (Christmas kitchen mat, not a desk) & $R$ & $\checkmark$ \\
2 & click[19.7x31.5in+19.7x63in] & $R$ & Attribute click on wrong product & $D$ & $\times$ \\
3 & click[19.7x31.5in+19.7x63in] & $R$ & Redundant repeat & $R$ & $\checkmark$ \\
4 & search[non slip desk christmasgoo3302 \ldots] & $E$ & Recovery attempt: re-search & $E$ & $\checkmark$ \\
5 & search[non slip desk color: christmasgoo3302 \ldots] & $E$ & Recovery attempt: refined search & $D$ & $\times$ \\
6 & click[b09cq45zrb] & $R$ & Re-clicks same wrong product & $R$ & $\checkmark$ \\
7 & click[christmasgoo3302] & $R$ & Attribute click on wrong product & $D$ & $\times$ \\
8 & click[christmasgoo3302] & $R$ & Redundant & $R$ & $\checkmark$ \\
9 & click[19.7x31.5in+19.7x63in] & $R$ & Attribute click on wrong product & $R$ & $\checkmark$ \\
10 & click[buy now] & $R$ & Purchases wrong product, reward stays 0 & $R$ & $\checkmark$ \\
\bottomrule
\end{tabular}
\end{table}

\subsection{Aggregate Observations}

Table~\ref{tab:appAggregate} summarizes the role distribution in the six audited trajectories.

\begin{table}[h]
\centering
\caption{Role distribution in logged GRPO rollouts (length-weighted mean over the audited ALFWorld and WebShop trajectories). Regression mass is high in both environments---exactly where \method{} yields its largest gains---and much of it sits inside successful trajectories that vanilla GRPO still reinforces.}
\label{tab:appAggregate}
\small
\setlength{\tabcolsep}{6pt}
\begin{tabular}{lcccc}
\toprule
Environment & $D$ & $E$ & $N$ & $R$ \\
\midrule
ALFWorld & 23\% & 26\% & 3\% & 48\% \\
WebShop  & 23\% & 23\% & 10\% & 43\% \\
\bottomrule
\end{tabular}
\end{table}

The main takeaway is that regression is common in these logged rollouts, especially as redundant repetition rather than irreversible state corruption. Several successful trajectories contain substantial $R$ mass, so vanilla GRPO would still broadcast positive credit to repeated inspections or repeated attribute clicks. This makes $R$-in-success the most important diagnostic cell for \method{} and motivates calibrating the role-conditioned mixing coefficient $\lambda$ on a small per-environment annotated sample.

\section{Judge Model Audit: Side-by-Side Hand vs Qwen3-8B-Thinking Labels}
\label{app:judge-audit}

\paragraph{Role-judge context window.}
\label{app:judge-context-window}

For a segment $k$, the training-time role judge sees a bounded local window around that segment. In our implementation, the window contains the task goal, up to five previous action--observation pairs $(a_{i,k-5},o_{i,k-5},\ldots,a_{i,k-1},o_{i,k-1})$, the current action $a_{i,k}$, the immediate resulting observation $o_{i,k}$, and up to five future action--observation pairs $(a_{i,k+1},o_{i,k+1},\ldots,a_{i,k+5},o_{i,k+5})$ when they exist. Boundary cases use the available prefix or suffix.

The short future window helps identify whether an exploratory segment enabled later progress or whether an apparently harmless step was redundant. We do not feed the entire trajectory to every segment-level judge call because long inputs make repeated high-quality judging expensive and empirically make the classifier less focused on the local causal role.

Controlling the input length keeps the role classifier usable at segment scale and reduces the chance that it relies on distant recovery patterns instead of the current action. The judge still does not receive the final verifier outcome or an unbounded future trajectory, so the role label diagnoses local causal behavior rather than copying the trajectory-level reward that GRPO already supplies.

\paragraph{Setup.}
We audit a Qwen3-8B judge with thinking mode enabled on 18 logged
trajectories (9 success, 9 failure) across three environments: 3 ALFWorld
(captured from the trained GRPO policy), 3 WebShop (trained policy), and
12 Search-QA (base-model rollouts to obtain failure-rich data). To keep the
ground truth independent of the rubric design, two annotators who did not
participate in defining the role taxonomy of Section~\ref{sec:taxonomy} each
labeled all 135 segments independently. The two annotators reached
\textbf{88.1}\% raw label agreement (119 of 135 segments); segments on which they disagreed were
adjudicated by a senior annotator, and the adjudicated labels are used as
ground truth. For each audited segment, the judge
was given the same bounded window used during training: the task, up to five
previous action--observation pairs, the current action and immediate
observation, and up to five future action--observation pairs when available.
The judge was not given the final verifier outcome or the unbounded full
trajectory. It was asked to output one role for the current segment using the
Qwen3 chat-template \texttt{enable\_thinking=True} flag. All inference used
temperature $0$.
Together with the merged ALFWorld and WebShop tables in Appendix~\ref{app:role-audit}, this appendix reports every trajectory with both hand and judge labels per step. Aggregate judge metrics are reported in Table~\ref{tab:judge-audit-f1}.

\paragraph{Judge prompt.}
The audit used the following role-classification prompt. We require a short evidence string for every segment, which forces the judge to ground each label in the local action--observation context rather than emitting only a free-floating role tag; in practice this makes label audits easier and improves judge consistency.

\begin{verbatim}
You are an expert evaluator of multi-turn agent trajectories.

You will see a local window around one target segment: the task, up to
five previous action-observation pairs, the CURRENT action and observation,
and up to five future action-observation pairs. Classify only the CURRENT
action into ONE of four roles:

    D (DECISIVE)    The action completes a required sub-goal or makes a
                                    verifier-checkable state change directly required by
                                    the task (e.g. takes the target object, performs a
                                    required transformation like cool/heat/clean, places
                                    the target in the destination, executes the final
                                    purchase, selects a task-mandated attribute).

    E (EXPLORATION) The action gathers information or visits a new state
                                    for the first time without completing a sub-goal.
                                    First-time inspection of a container, first navigation
                                    to a candidate location, an initial search query,
                                    a refined search after recognizing a wrong commitment.

    N (NO-PROGRESS) The action neither changes the task state nor reveals
                                    new information. Empty-handed traversal, harmless
                                    duplicate after task completion, generic navigation
                                    through an irrelevant location with no investigation.

    R (REGRESSION)  Clear setback: the action either corrupts state,
                                    picks the wrong object, commits to a non-matching
                                    product, performs the wrong transformation, OR is a
                                    redundant repeat of an already-completed action that
                                    yields no new information ("examine X" when X was just
                                    examined; re-click of an already-selected attribute;
                                    re-purchase after success).

CALIBRATION RULES
    - Judge LOCAL causal role using only the supplied window. Do not infer
        credit from distant recovery or distant failure outside the window.
    - For the current step, provide brief evidence grounded in the local
        action/observation, e.g. "first inspection reveals new object",
        "repeat with no new information", or "correct target acquired".
    - First-time examine/inspect = E. Second-time examine of the same
        target without state change = R.
    - "Nothing happens." in observation means the action was invalid;
        if action repeats, label R.
    - A buy/place/take/heat/cool of the correct target = D.
    - Re-click of already-selected attribute = R, even if the local observation
        reports success.

OUTPUT FORMAT
After your reasoning, output ONLY a JSON object on a single line at
the very end:
{"labels": ["D"|"E"|"N"|"R", ...], "evidence": ["short reason per step", ...]}
Both lists must have length equal to the number of steps shown.
\end{verbatim}

\paragraph{ALFWorld and WebShop trajectories.}
The six ALFWorld and WebShop trajectories audited here (A1--A3, W1--W3) are
the same rollouts analyzed in Appendix~\ref{app:role-audit}. To avoid
duplicating their per-step action listings, their hand labels, Qwen3-8B-thinking
judge labels, and per-step agreement are reported together with the
role-distribution analysis in Tables~\ref{tab:appA1}--\ref{tab:appW3} (judge
agreement per trajectory is stated in each caption). The Search-QA
trajectories below are audited only here.

\paragraph{Search-QA trajectory summary.}
Table~\ref{tab:appJ-SQ-summary} summarizes all 12 Search-QA audit trajectories. The table keeps the outcome, question, number of search turns, final answer, hand-label sequence, judge-label sequence, and agreement count; Table~\ref{tab:appJ-SQ-F5} then gives the only Search-QA disagreement case step by step.

\begin{table}[h]
\centering
\scriptsize
\setlength{\tabcolsep}{2pt}
\caption{Search-QA audit summary. Label sequences are ordered by environment-facing segment; $S^m\!\to\!A$ denotes $m$ search turns followed by one answer turn.}
\label{tab:appJ-SQ-summary}
\begin{tabular}{@{}p{0.08\linewidth}p{0.07\linewidth}p{0.33\linewidth}p{0.09\linewidth}p{0.11\linewidth}p{0.11\linewidth}p{0.11\linewidth}@{}}
\toprule
ID & Outcome & Question / final answer & Pattern & Hand labels & Judge labels & Agreement \\
\midrule
SQ-F1 & fail & first Nobel Prize in Physics / Wilhelm Röntgen & $S^2\to A$ & E,E,R & E,E,R & 3/3 \\
SQ-F2 & fail & next Deadpool movie release / Deadpool 3 & $S^3\to A$ & E,E,E,R & E,E,E,R & 4/4 \\
SQ-F3 & fail & short-wave broadcast mode / AM & $S^3\to A$ & E,E,E,R & E,E,E,R & 4/4 \\
SQ-F4 & fail & southwest wind across Nigeria / February and June & $S^3\to A$ & E,E,E,R & E,E,E,R & 4/4 \\
SQ-F5 & fail & first declaration of human rights / John Peters Humphrey & $S^3\to A$ & E,E,R,R & E,E,E,R & 3/4 \\
SQ-F6 & fail & next \emph{Scandal} episode / April 19, 2018 & $S^3\to A$ & E,E,R,R & E,E,R,R & 4/4 \\
SQ-F7 & fail & Philadelphia last Super Bowl win / 2018 & $S^2\to A$ & E,E,R & E,E,R & 3/3 \\
SQ-F8 & fail & first lady nominated to Rajya Sabha / Rajvanshi Devi & $S^3\to A$ & E,R,R,R & E,R,R,R & 4/4 \\
SQ-S1 & success & Swan Lake, Sleeping Beauty, Nutcracker composer / Pyotr Ilyich Tchaikovsky & $S^2\to A$ & E,E,D & E,E,D & 3/3 \\
SQ-S2 & success & Dragon Ball Z episode count / 291 & $S^2\to A$ & E,E,D & E,E,D & 3/3 \\
SQ-S3 & success & garden city of New Earswick designer / Raymond Unwin & $S^3\to A$ & E,E,E,D & E,E,E,D & 4/4 \\
SQ-S4 & success & filming location of \emph{The Curse of Oak Island} / Oak Island & $S^2\to A$ & E,E,D & E,E,D & 3/3 \\
\bottomrule
\end{tabular}
\end{table}

\paragraph{Representative Search-QA disagreement.}
SQ-F5 is the only Search-QA trajectory in this audit where Qwen3-8B-thinking disagrees with the adjudicated hand labels. The disagreement is instructive: the third search shifts from the human-rights question to a declaration-of-independence query, so annotators mark it as regression, while the judge still treats it as exploration.

\begin{table}[h]
\centering
\footnotesize
\caption{Trajectory SQ-F5: representative Search-QA disagreement.}
\label{tab:appJ-SQ-F5}
\begin{tabular}{@{}cp{0.58\linewidth}ccc@{}}
\toprule
$t$ & Action & Hand & Qwen3-8B-think & Agree \\
\midrule
0 & \textless{}search\textgreater{}who wrote the first declaration of human rights\textless{}/sea & $E$ & $E$ & $\checkmark$ \\
1 & \textless{}search\textgreater{}who wrote the first declaration of human rights decl & $E$ & $E$ & $\checkmark$ \\
2 & \textless{}search\textgreater{}who wrote the declaration of independence\textless{}/search\textgreater{} & $R$ & $E$ & $\times$ \\
3 & \textless{}answer\textgreater{}John Peters Humphrey\textless{}/answer\textgreater{} & $R$ & $R$ & $\checkmark$ \\
\bottomrule
\end{tabular}
\end{table}

\paragraph{Audit blind spot.}
None of the nine failure trajectories in this set contain any
hand-labeled $D$ segment. This is a structural property of the
calibration set rather than a sampling artifact: WebShop W3 commits
to the wrong product at step~1 and accumulates only $R$ thereafter,
and the eight Search-QA failures all terminate with a wrong
$\langle$answer$\rangle$ ($R$) after a sequence of $\langle$search$\rangle$
queries ($E$). The cell ``$D$ in failed rollouts'' is therefore not measurable
on this calibration set. Verifying that the judge correctly
identifies decisive intermediate progress \emph{within failed
trajectories} (for example, an ALFWorld agent that correctly heats
the target object but then places it in the wrong receptacle, or a
Search-QA agent that correctly identifies the bridge entity but
issues a malformed final answer) is the principal extension
required of a larger follow-up audit.

\section{Verbatim Search-QA Query Reformulations}
\label{app:sq-queries}

For completeness, this section restores the exact per-step text behind the
compressed label sequences of Table~\ref{tab:appJ-SQ-summary}. For every one
of the twelve audited Search-QA trajectories we list each environment-facing
segment in order: search-turn queries followed by the final answer turn,
reproduced verbatim from the logged rollout (a few strings were truncated at
capture time and are marked with ``\ldots''). Each segment is tagged with its
adjudicated hand role $[D/E/N/R]$; the single judge disagreement in this set is
SQ-F5 segment~2, where the judge predicts $E$ instead of the adjudicated $R$,
marked $[R\,|\,\text{judge }E]$. Segments separated by ``$\;\mid\;$''; the
answer turn is prefixed by $a$.

\paragraph{Failure trajectories.}
\begin{description}[leftmargin=2.2em,style=nextline,itemsep=2pt,topsep=2pt]
\item[SQ-F1] (answer ``Wilhelm Röntgen'')
$s_0$\,[E] first nobel prize in physics $\;\mid\;$ $s_1$\,[E] Wilhelm Röntgen first Nobel Prize in Physics $\;\mid\;$ $a$\,[R] Wilhelm Röntgen.
\item[SQ-F2] (answer ``Deadpool 3'')
$s_0$\,[E] next Deadpool movie release $\;\mid\;$ $s_1$\,[E] next Deadpool movie release date $\;\mid\;$ $s_2$\,[E] Deadpool 3 movie release date $\;\mid\;$ $a$\,[R] Deadpool 3.
\item[SQ-F3] (answer ``AM'')
$s_0$\,[E] short wave broadcast service $\;\mid\;$ $s_1$\,[E] short wave broadcast service am $\;\mid\;$ $s_2$\,[E] short wave broadcast service mode $\;\mid\;$ $a$\,[R] AM.
\item[SQ-F4] (answer ``February and June'')
$s_0$\,[E] south west wind blows across Nigeria $\;\mid\;$ $s_1$\,[E] south west wind blows across Nigeria between $\;\mid\;$ $s_2$\,[E] south west wind blows across Nigeria between Februar\ldots $\;\mid\;$ $a$\,[R] February and June.
\item[SQ-F5] (answer ``John Peters Humphrey'')
$s_0$\,[E] who wrote the first declaration of human rights $\;\mid\;$ $s_1$\,[E] who wrote the first declaration of human rights decl\ldots $\;\mid\;$ $s_2$\,$[R\,|\,\text{judge }E]$ who wrote the declaration of independence $\;\mid\;$ $a$\,[R] John Peters Humphrey.
\item[SQ-F6] (answer ``April 19, 2018'')
$s_0$\,[E] ``Scandal'' TV show $\;\mid\;$ $s_1$\,[E] Scandal next episode $\;\mid\;$ $s_2$\,[R] Scandal next episode (verbatim repeat of $s_1$) $\;\mid\;$ $a$\,[R] April 19, 2018.
\item[SQ-F7] (answer ``2018'')
$s_0$\,[E] Philadelphia last Super Bowl win $\;\mid\;$ $s_1$\,[E] Philadelphia Eagles last Super Bowl win $\;\mid\;$ $a$\,[R] 2018.
\item[SQ-F8] (answer ``Rajvanshi Devi'')
$s_0$\,[E] first lady nominated member of the rajya sabha $\;\mid\;$ $s_1$\,[R] first lady nominated member of the rajya sabha (repeat) $\;\mid\;$ $s_2$\,[R] first lady nominated member of the rajya sabha (repeat) $\;\mid\;$ $a$\,[R] Rajvanshi Devi.
\end{description}

\paragraph{Success trajectories.}
\begin{description}[leftmargin=2.2em,style=nextline,itemsep=2pt,topsep=2pt]
\item[SQ-S1] (answer ``Pyotr Ilyich Tchaikovsky'')
$s_0$\,[E] ``Swan Lake the Sleeping Beauty and the Nutcracker fa\ldots $\;\mid\;$ $s_1$\,[E] Swan Lake the Sleeping Beauty and the Nutcracker fam\ldots $\;\mid\;$ $a$\,[D] Pyotr Ilyich Tchaikovsky.
\item[SQ-S2] (answer ``291'')
$s_0$\,[E] Dragon Ball Z episodes $\;\mid\;$ $s_1$\,[E] Dragon Ball Z episodes 291 $\;\mid\;$ $a$\,[D] 291.
\item[SQ-S3] (answer ``Raymond Unwin'')
$s_0$\,[E] garden city of new earswick $\;\mid\;$ $s_1$\,[E] who designed the model village of new earswick $\;\mid\;$ $s_2$\,[E] Raymond Unwin designed the garden city of new earswi\ldots $\;\mid\;$ $a$\,[D] Raymond Unwin.
\item[SQ-S4] (answer ``Oak Island'')
$s_0$\,[E] ``The Curse of Oak Island filmed'' $\;\mid\;$ $s_1$\,[E] where is the tv show the curse of oak island filmed $\;\mid\;$ $a$\,[D] Oak Island.
\end{description}

\end{document}